%% file: main.tex
\documentclass[11pt]{article} 
\pdfoutput=1
\usepackage[margin=1in]{geometry}

\usepackage{natbib}
\usepackage{algorithm}
\usepackage{algorithmic}
\usepackage{graphicx}
\usepackage{balance} 
\usepackage{mathtools}  
\usepackage{tabulary}
\usepackage{booktabs}
\usepackage{microtype}
\usepackage{graphicx}
\usepackage{subfigure}
\usepackage{booktabs} 

\usepackage{hyperref}

\usepackage{amsmath}
\usepackage{amssymb}
\usepackage{mathtools}
\usepackage{amsthm}

\usepackage{makecell}

\newcommand\A{\mathsf{A}}
\newcommand\C{\mathcal{C}}
\newcommand\D{\mathcal{D}}
\newcommand\HH{\mathcal{H}}
\newcommand\PP{\mathcal{P}}
\newcommand\X{\mathcal{X}}

\newcommand\Z{\mathcal{Z}}
\newcommand\x{\boldsymbol{x}}
\newcommand\z{\boldsymbol{z}}
\newcommand\err{\mathrm{err}}
\newcommand\dis{\mathrm{dis}}
\newcommand\repd{\mathrm{RepD}}
\newcommand\vc{\mathrm{VC}}
\newcommand\bin{\mathrm{Bin}}
\newcommand\lap{\mathrm{Lap}}
\newcommand\ind{\mathbf{1}}

\usepackage[capitalize,noabbrev]{cleveref}

\theoremstyle{plain}
\newtheorem{theorem}{Theorem}[section]

\newtheorem{lemma}[theorem]{Lemma}
\newtheorem{corollary}[theorem]{Corollary}
\theoremstyle{definition}
\newtheorem{definition}[theorem]{Definition}

\theoremstyle{remark}

\usepackage[textsize=tiny]{todonotes}

\title{Improved Bounds for Pure Private Agnostic Learning: Item-Level and User-Level Privacy}

\author{
Bo Li\thanks{Department of Computer Science and Engineering, HKUST. \texttt{bli@cse.ust.hk}.}
\and 
Wei Wang\thanks{Department of Computer Science and Engineering, HKUST. \texttt{weiwa@cse.ust.hk}.}
\and
Peng Ye\thanks{Department of Computer Science and Engineering, HKUST. \texttt{pyeac@connect.ust.hk}.}
}

\begin{document}

\maketitle

\begin{abstract}
    Machine Learning has made remarkable progress in a wide range of fields. In many scenarios, learning is performed on datasets involving sensitive information, in which privacy protection is essential for learning algorithms. In this work, we study pure private learning in the agnostic model -- a framework reflecting the learning process in practice. We examine the number of users required under item-level (where each user contributes one example) and user-level (where each user contributes multiple examples) privacy and derive several improved upper bounds. For item-level privacy, our algorithm achieves a near optimal bound for general concept classes. We extend this to the user-level setting, rendering a tighter upper bound than the one proved by~\citet{ghazi2023user}. Lastly, we consider the problem of learning thresholds under user-level privacy and present an algorithm with a nearly tight user complexity.
\end{abstract}

\input{text/intro.tex}
\input{text/preliminaries.tex}
\input{text/item.tex}
\input{text/user.tex}
\input{text/threshold.tex}
\input{text/conclusion.tex}

\section*{Acknowledgements}

The research was supported in part by a RGC RIF grant under the contract R6021-20, RGC CRF grants under the contracts C7004-22G and C1029-22G, and RGC GRF grants under the contracts 16209120, 16200221, 16207922, and 16211123. We would like to thank the anonymous reviewers for their helpful comments.



\bibliography{references.bib}
\bibliographystyle{unsrtnat}

\newpage
\appendix
\onecolumn
\input{appendix/preliminaries.tex}
\input{appendix/item.tex}
\input{appendix/user.tex}
\input{appendix/threshold.tex}

\end{document}

%% file: text/intro.tex
\section{Introduction}

Differential Privacy (DP)~\citep{dwork2006calibrating,dwork2006our} is a mathematical definition for measuring the privacy of algorithms. An algorithm is considered private if the presence or absence of a single user does not significantly affect the output. Due to its soundness and quantifiability, DP has become the gold standard for ensuring data privacy and has been employed by the industry~\citep{app2017apple} and the governments~\citep{abowd2018us}.

Machine learning models are usually trained on datasets that contain sensitive data (e.g., in medical or financial applications). Thus, it is necessary to design privacy-preserving machine learning algorithms.~\citet{kasiviswanathan2011can} initiated the study of private learning and defined private PAC learning as a combination of probably approximately correct (PAC) learning~\citep{valiant1984theory} and differential privacy. The most important measure for the utility of a learner is the minimum amount of data required to find a hypothesis that achieves some target accuracy, which is called the sample complexity. Subsequent work (e.g.,~\citep{feldman2014sample,alon2019private}) showed that the privacy constraint makes the sample complexity much higher than the non-private setting.

The PAC model assumes that all the samples are labeled by a concept in the given concept class. This is called the realizable assumption. However, in many situations, such an assumption could be unrealistic. For example, the collected data may be noisy, or one may choose an inappropriate concept class that cannot classify the samples perfectly. The agnostic learning model, introduced by~\citet{haussler1992decision} and~\citet{kearns1994toward}, addresses this issue by requiring the learner to output a hypothesis within some additive error compared to the best one in the concept class. Motivated by this,~\citet{kasiviswanathan2011can} also defined private agnostic learning with the agnostic model replacing the PAC model in private PAC learning.

For pure private realizable learning, tight sample complexity bounds were shown by~\citet{beimel2019characterizing}. However, for agnostic learning, there were no such tight results. Though several algorithms were proposed~\citep{kasiviswanathan2011can,beimel2014learning,alon2020closure}, none of them achieve a sample complexity that matches the trivial lower bound, i.e., the one combines lower bounds of private realizable learning and non-private agnostic learning. On the other hand, there have been no non-trivial lower bounds obtained in the literature. Thus, a gap exists, leading to the following question:
\begin{quote}
\centering
\emph{What is the sample complexity of pure private agnostic learning?}
\end{quote}

What we have discussed so far only considers the situation in that each user contributes one example, which we refer to as \emph{item-level} DP. In practice, however, one user may have many items (e.g., in federated learning~\citep{kairouz2021advances}). Here, the goal becomes protecting all the examples contributed by a single user. In this \emph{user-level} DP setting, we are more interested in the user complexity of learning algorithms, i.e., the number of users needed to achieve the target accuracy.

User-level pure private PAC learning was first studied by~\citet{ghazi2021user}, who showed that compared to the item-level setting, one could learn with fewer users if each user contributes sufficiently many examples. The work of~\citet{ghazi2023user} further tightly determined the user complexity as a function of the number of examples per user. They also proved upper and lower bounds for the agnostic model. However, like the item-level setting, there is a gap between the upper and lower bounds. Thus, it is natural to ask:
\begin{quote}
\centering
    \emph{Can we obtain tighter bounds on the user complexity of pure private agnostic learning?}
\end{quote}

\subsection{Our Results}

In this work, we study the problem of agnostic learning with pure differential privacy. We list our results as follows. The formal definitions of our notions can be found in Section~\ref{sec:pre}.

\paragraph*{Item-Level DP.} We show that every concept class $\C$ can be learned up to excess error $\alpha$ by an $\varepsilon$-differentially private algorithm using $\widetilde{O}\left(\frac{\repd(\C)}{\alpha\varepsilon} + \frac{\vc(\C)}{\alpha^2}\right)$ samples. Compared with private realizable learning~\citep{beimel2019characterizing}, our result only incurs an additive term of $\widetilde{O}\left(\frac{\vc(\C)}{\alpha^2}\right)$, which is known to be tight up to polylogarithmic factors even for non-private agnostic learning. In contrast, previous results~\citep{kasiviswanathan2011can,beimel2019characterizing,alon2020closure} incur higher sample complexity than ours. 

\paragraph*{User-Level DP.} We provide a generic learner that agnostically learns any $\C$ using $\widetilde{O}\Big(\frac{\repd(\C)}{\varepsilon} +  \frac{\repd(\C)}{\sqrt{m}\alpha\varepsilon} + \frac{\vc(\C)}{m\alpha^2}\Big)$ users, where $m$ is the number of examples held by each user. This improves the previous result of~\citet{ghazi2023user}, wherein the third term is $\widetilde{O}\left(\frac{\repd(\C)}{m\alpha^2}\right)$. As indicated by their lower bound, the third term in our upper bound is optimal while a gap still exists regarding the second term.

\paragraph*{Thresholds with User-Level DP.} For the case that $\C$ is the class of thresholds over some domain $\X$, we prove that the user complexity upper bound can be further improved to $\widetilde{O}\Big(\frac{\log\vert\X\vert}{\varepsilon} + \frac{\log\vert\X\vert}{m\alpha\varepsilon} +\frac{1}{\sqrt{m}\alpha\varepsilon} + \frac{1}{m\alpha^2}\Big)$, matching the lower bound proved by~\citet{ghazi2023user}\footnote{In their work, they claimed that their algorithm can achieve this complexity for thresholds. We have confirmed with them that the claim made a mistake -- their algorithm still has a user complexity of $\widetilde{O}\left(\frac{\log\vert\X\vert}{\varepsilon} +\frac{\log\vert\X\vert}{\sqrt{m}\alpha\varepsilon} + \frac{1}{m\alpha^2}\right)$, which is not optimal.}.

The improvement made by our results is significant in the high-accuracy regime. To see this, fix $m, \varepsilon$ and let $\alpha\to 0$. Our upper bound is dominated by $\widetilde{O}\left(\frac{\vc(\C)}{m\alpha^2}\right)$, which even matches the cost of non-private agnostic learning (including the item-level setting, where $m = 1$). In other words, privacy is free -- the privacy constraint can be satisfied with the same number of users. In comparison, the number of users required by the algorithm of~\citet{ghazi2023user} is $\widetilde{O}\left(\frac{\repd(\C)}{m\alpha^2}\right)$, which can be much larger than ours as suggested by the results of~\citet{feldman2014sample}.

We summarize the previous and new bounds in Table~\ref{tab:sum}.

\begin{table*}[t]
\caption{Summary of our results and previous results. For user-level DP, we assume $m \le 1/\alpha^2$ for simplicity, since as shown by~\citet{ghazi2023user}, increasing $m$ won't make the user complexity decrease when $m > 1/\alpha^2$.}
\label{tab:sum}
\vskip 0.15in
\begin{center}
\begin{small}
\resizebox{\columnwidth}{!}{%
\begin{tabular}{ccccc}
\toprule
Concept Class & DP & Lower Bound& Previous Upper Bound & Our Upper Bound \\
\midrule
Arbitrary $\C$   & Item & \makecell{$\widetilde{\Omega}\left(\frac{\repd(\C)}{\alpha\varepsilon} + \frac{\vc(\C)}{\alpha^2}\right)$ \\\citep{simon1996general}\\and~\citep{beimel2019characterizing}}& \makecell{$\widetilde{O}\left(\frac{\repd(\C)}{\alpha\varepsilon} + \frac{\repd(\C)}{\alpha^2}\right)$ \\\citep{beimel2019characterizing}\\ or $\widetilde{O}\left(\frac{\repd(\C)}{\alpha\varepsilon} + \frac{\vc(\C)}{\alpha^2\varepsilon}\right)$\\\citep{alon2020closure}} &  \makecell{$\widetilde{O}\left(\frac{\repd(\C)}{\alpha\varepsilon} + \frac{\vc(\C)}{\alpha^2}\right)$ \\(Theorem~\ref{thm:item})}\\\midrule
Arbitrary $\C$   & User & \makecell{$\widetilde{\Omega}\left(\frac{\repd(\C)}{\min(1,m\alpha)\varepsilon} + \frac{\vc(\C)}{\sqrt{m}\alpha\varepsilon} + \frac{\vc(\C)}{m\alpha^2}\right)$\\\citep{ghazi2023user}} & \makecell{$\widetilde{O}\left(\frac{\repd(\C)}{\sqrt{m}\alpha\varepsilon} + \frac{\repd(\C)}{m\alpha^2} \right)$\\\citep{ghazi2023user}}&  \makecell{$\widetilde{O}\left(\frac{\repd(\C)}{\sqrt{m}\alpha\varepsilon} + \frac{\vc(\C)}{m\alpha^2} \right)$ \\(Theorem~\ref{thm:user})} \\\midrule
\makecell{Thresholds \\over $\X$} & User & \makecell{$\widetilde{\Omega}\left(\frac{\log\vert\X\vert}{\min(1,m\alpha)\varepsilon} +\frac{1}{\sqrt{m}\alpha\varepsilon} + \frac{1}{m\alpha^2}\right)$\\\citep{ghazi2023user}}& \makecell{$\widetilde{O}\left(\frac{\log\vert\X\vert}{\sqrt{m}\alpha\varepsilon} + \frac{1}{m\alpha^2}\right)$\\\citep{ghazi2023user}} & \makecell{$\widetilde{O}\left(\frac{\log\vert\X\vert}{\min(1, m\alpha)\varepsilon} +\frac{1}{\sqrt{m}\alpha\varepsilon} + \frac{1}{m\alpha^2}\right)$ \\(Theorem~\ref{thm:threshold})} \\
\bottomrule
\end{tabular}
}
\end{small}
\end{center}
\vskip -0.1in
\end{table*}

\subsection{Technical Overview}

We now provide a rough overview of our proofs.

\paragraph*{Item-Level DP.} We start by describing the algorithm proposed by~\citet{beimel2019characterizing}. Their algorithm first samples a hypothesis class $\HH$ with $\log\vert\HH\vert = \widetilde{O}\left(\repd(\C)\right)$. Then, it applies the exponential mechanism, with the score function being the empirical error. By the utility property of the exponential mechanism, it returns some hypothesis with a small empirical error. To ensure that the generalization error of the output hypothesis is also small, we have to prove agnostic generalization for $\HH$. This requires $\widetilde{O}\left(\frac{\repd(\C)}{\alpha^2}\right)$ samples, which is higher than our goal.

To reduce the sample complexity, we construct a surrogate error as the score function. Consider a hypothesis $h\in\HH$. Note that by triangle inequality, for any concept $c$, the error of $h$ is no more than the error of $c$ plus the disagreement between $c$ and $h$. Thus, by taking a minimum over $c\in\C$, we get an upper bound on the error of $h$. This quantity is the surrogate error of $h$. We then use the empirical surrogate error as the score function of $h$.

The exponential mechanism finds a hypothesis $h_0$ with a low empirical surrogate error. By our construction, we can prove that there exists some $c_0$ s.t. both the empirical error of $c_0$ and the empirical disagreement between $c_0$ and $h_0$ are small. To ensure they are also small on the distribution, we need to show agnostic generalization for $\C$ and realizable generalization for $\C\times\HH$. The former needs $\widetilde{O}\left(\frac{\vc(\C)}{\alpha^2}\right)$ samples, reducing the $\repd(\C)$ factor to $\vc(\C)$. The latter can be satisfied by $\widetilde{O}\left(\frac{\repd(\C)}{\alpha}\right)$ samples, saving a $1/\alpha$ multiplicative factor. We thus achieve our goal.

\paragraph*{User-Level DP.} We extend our item-level algorithm to the user-level. The high-level idea relies on a key observation from~\citep{adell2006exact,liu2020learning}, which states that the total variation distance between two binomial distributions $\bin(m, p)$ and $\bin(m, q)$ scales as $\sqrt{m}|p - q|$. We define a notion called user-level error and use the observation to show that the user-level error amplifies the item-level error by $\sqrt{m}$ when there are $m$ examples per user. That is, suppose hypothesis $h$ has an additive error of $\alpha$ worse than the best concept in $\C$, its user-level error must be larger than the best by $\Omega(\sqrt{m}\alpha)$.

Thus, we can then apply our item-level algorithm with the item-level error replaced by the user-level error. For a target error $\alpha$, we only have to find a hypothesis with a user-level error of $O(\sqrt{m}\alpha)$. This explains why fewer users are required in the user-level setting.

There are other technical issues in directly applying the item-level algorithm, such as user-level generalization lemmata and estimation of the minimum error. We develop some techniques to handle these. The details of these are presented in Section~\ref{sec:user}.

\paragraph*{Thresholds with User-Level DP.} Our algorithm for learning thresholds employs binary search. We start with the entire domain. In each iteration, we select the median w.r.t. the distribution of points in the current interval to split it into two parts. We add Laplace noise to the minimum user-level error of thresholds in each part and choose the one with a smaller error to continue. After $O(\log(1/\alpha))$ iterations, we will reach an interval with density at most $O(\alpha)$. By our user-level error results and the utility guarantee of the Laplace mechanism, $\widetilde{O}\left(\frac{1}{\sqrt{m}\alpha\varepsilon}\right)$ users are sufficient for any threshold in the final interval to have an additive error of $\alpha$.

To select the median, we exploit an observation that the user-level setting amplifies the density by $m$. We use this to design an algorithm that approximates the median with a user complexity of $\widetilde{O}\left(\frac{\log\vert\X\vert}{m\alpha\varepsilon}\right)$, which matches our target.

\subsection{Related Work}

As the initial work of private learning,~\citet{kasiviswanathan2011can} defined the notion of private learning and gave algorithms for finite concept classes.~\citet{beimel2014bounds} showed that properly learning point functions under pure DP requires much more samples than non-private learning, indicating that the structure of learning becomes very different under the privacy restrictions. They also proposed the idea of representation to devise algorithms for pure private improper learning. This idea was further developed by the work of~\citet{beimel2019characterizing}, which defined the notion of probabilistic representation dimension and showed that it characterizes the sample complexity of (improper) pure private learners.~\citet{feldman2014sample} showed an equivalence between the representation dimension and the communication complexity of the evaluation problem and used this relation to separate the sample complexity of pure private learning from non-private learning. For agnostic learning,~\citet{beimel2014learning} proposed a realizable-to-agnostic transformation that works for proper learners. Based on this result,~\citet{alon2020closure} proposed a more general one, which works for improper learners.

Though this work only considers pure DP, it is worth mentioning that there were also a great number of results on learning under approximate DP. It is a relaxed notion of pure DP that allows privacy to be violated with a negligible probability of $\delta$. The most remarkable result is the equivalence between approximate private learning and (non-private) online learning~\citep{alon2019private,bun2020equivalence,ghazi2021sample,alon2022private}, which demonstrates a strong connection between these two tasks.

The study of user-level private learning PAC learnable classes was initiated by~\citet{ghazi2021user}. They showed that for both pure and approximate DP, if each user contributes sufficiently many examples, the learning task can be done with much fewer users than in the item-level DP setting.~\citet{ghazi2023user} improved this result by representing the user complexity as a function of the number of samples per user. For pure DP, they showed the exact user complexity of private PAC learning and established non-trivial upper and lower bounds for agnostic learning. For approximate DP, they proposed a general transformation that can convert any item-level DP algorithm for any statistical tasks (including but not limited to learning) to a user-level one and yield a multiplicative saving of $\sqrt{m}$ on the number of users.

%% file: text/preliminaries.tex
\section{Preliminaries}
\label{sec:pre}

\paragraph*{Notation.} We use $\widetilde{O},\widetilde{\Omega}$ and $\widetilde{\Theta}$ to hide $\mathrm{polylog}(1/\alpha, 1/\beta)$ factors. Throughout the paper, we use $m$ to denote the number of samples per user. The item-level setting refers to the case that $m = 1$.

We first recall the notion of DP. A dataset can be represented as $\z=(z_{1,1},\dots,z_{1,m},z_{2,1},\dots,\allowbreak z_{n,m})\in\Z^{nm}$, where $\z_i=(z_{i,1},\dots,z_{i,m})$ is the set of samples contributed by the $i$-th user. Two datasets are neighboring if one can be obtained from the other by adding or removing a single user.

\begin{definition}[Differential Privacy~\citep{dwork2006calibrating,dwork2006our}]
\label{def:dp}
    A randomized algorithm $\A$ is $\varepsilon$-differentially private if for any neighboring datasets $\z$ and $\z'$, and any subset $\mathcal{O}$ of outputs, it holds that $\Pr[\A(\z)\in\mathcal{O}]\le e^{\varepsilon}\Pr[\A(\z')\in\mathcal{O}]$.
\end{definition}

\subsection{Learning}

Let $\X$ be an arbitrary domain\footnote{We assume $\X$ is countable to avoid making technical measurability assumptions.} and $\Z = \X\times\{0, 1\}$ be the domain of data samples. A concept $c:\X\to\{0, 1\}$ is a function that labels each (unlabeled) sample taken from $\X$ by either $0$ or $1$. A concept class $\C$ is a set of concepts over $\X$. The VC dimension of $\C$ is denoted by $\vc(\C)$. We use $\D$ to denote a distribution over $\Z$, and $\D_{\X}$ to denote its marginal distribution over $\X$. 

In this work, we focus on the agnostic learning model~\citep{haussler1992decision, kearns1994toward}, where the data distribution can be arbitrary, and the goal of the learning algorithm is to produce a hypothesis whose generalization error is close to the best possible (by concepts in $\C$). The generalization error of a hypothesis $h: \X\to\{0, 1\}$ with respect to a distribution $\D$ over $\Z$ is defined as $\err_{\D}(h) = \Pr_{(x, y)\sim\D} [h(x)\neq y]$. 
\begin{definition}[Agnostic Learning]
\label{def:agn}
    We say a learning algorithm $\A$ is an $(\alpha, \beta)$-agnostic learner for concept class $\C$ if for any distribution $\D$ on $\X\times\{0, 1\}$, it takes in a dataset $\z$, where each $z_{i,j}$ is drawn i.i.d. from $\D$, and outputs a hypothesis $h$ such that
    \begin{equation*}
        \err_{\D}(h) \le \inf_{c\in\C}\err_{\D}(c) + \alpha
    \end{equation*}
    with probability at least $1-\beta$. The probability is over the random choice of the samples and the coin tosses of $\A$.
\end{definition}

In the standard PAC learning model~\citep{valiant1984theory}, it is assumed that there is some $c\in\C$ s.t. $\err_{\D}(c) = 0$. We call this realizable learning. Another way to describe the goal of realizable learning is finding a hypothesis $h$ with small generalization disagreement between $h$ and $c$ over $\D_\X$, which is defined as $\dis_{\D_\X}(c, h) = \Pr_{x\sim \D_\X}[c(x)\neq h(x)]$. The notion of disagreement plays a crucial role in our proofs.

\subsection{Probabilistic Representation Dimension}

The probabilistic representation dimension is a combinatorial parameter introduced by~\citet{beimel2019characterizing} to characterize the sample complexity of pure private realizable learner. Let $\PP$ be a distribution on hypothesis classes. Define the size of $\PP$ to be $\mathrm{size}(\PP) = \max_{\HH\in\mathrm{supp}(\PP)}\ln|\HH|$.

\begin{definition}[Probabilistic Representation Dimension]
    A distribution $\PP$ on hypothesis classes is said to be an $(\alpha, \beta)$-probabilistic representation of a concept class $\C$ if for any $c\in \C$ and any distribution $\D_{\X}$ on $\X$, with probability $1-\beta$ over $\HH\sim\PP$, there exists $h\in\HH$ such that $\dis_{\D_{\X}}(c, h) \le \alpha$.
    The $(\alpha, \beta)$-probabilistic representation dimension of a concept class $\C$ is defined as
    \begin{equation*}
        \repd_{\alpha,\beta}(\C)=\min_{\PP\textup{ is an }(\alpha, \beta)\textup{-probabilistic representation of }\C}\mathrm{size}(\PP).
    \end{equation*}
    Moreover, the probabilistic representation dimension of $\C$ is defined by taking $\alpha=\beta=1/4$:
    \begin{equation*}
        \repd(\C) = \repd_{1/4, 1/4}(\C).
    \end{equation*}
\end{definition}

The following lemma~\citep{beimel2019characterizing} shows that a probabilistic representation with arbitrary $\alpha$ and $\beta$ can be constructed from one with $\alpha = \beta=1/4$.

\begin{lemma}[Boosting Probabilistic Representation]
\label{lem:boosting}
    For any concept class $\C$, we have $\repd_{\alpha,\beta}(\C) \allowbreak= O(\log(1/\alpha)\cdot(\repd(\C) + \log\log\log(1/\alpha) + \log\log(1/\beta)))$ for $0<\alpha,\beta<1$.
\end{lemma}

\subsection{Tools from Differential Privacy}

We introduce some useful tools for achieving differential privacy. We say a function $f:\Z^* \to \mathbb{R}$ has sensitivity $\Delta$ if $|f(\z) - f(\z')| \le \Delta$ for all neighboring datasets $\z$ and $\z'$. Let $\lap(b)$ denote the Laplace distribution with mean $0$ and scale $b$. The Laplace mechanism is an algorithm that outputs $f(\z) + r$ where $r\sim\lap(\Delta/\varepsilon)$.

\begin{lemma}[The Laplace Mechanism]
\label{lem:lap}
    The Laplace mechanism is $\varepsilon$-differentially private. Moreover, it holds that $|r| \le \ln(1/\beta)\Delta/\varepsilon$ with probability at least $1-\beta$, where $r\sim\lap(\Delta/\varepsilon)$
\end{lemma}

We next describe the exponential mechanism~\citep{mcsherry2007mechanism}, which has been widely used in designing differentially private learning algorithms (e.g.,~\citep{kasiviswanathan2011can,beimel2019characterizing,alon2020closure}). Let $\HH$ be a finite set and $q:\Z^*\times \HH\to \mathbb{R}$ be a score function. We say $q$ has sensitivity $\Delta$ if $\max_{h\in\HH}|q(\z, h) - q(\z', h)|\le \Delta$ for any neighboring datasets $\z$ and $\z'$. The exponential mechanism outputs an $h\in\HH$ with probability
\begin{equation*}
    \frac{\exp(-\varepsilon\cdot q(\z, h)/2\Delta)}{\sum_{f\in\HH}\exp(-\varepsilon\cdot q(\z, f)/2\Delta)}.
\end{equation*}
\begin{lemma}[The Exponential Mechanism]
\label{lem:exp}
The exponential mechanism is $\varepsilon$-differentially private. Moreover, it outputs an $h$ s.t. 
\begin{equation*}
    q(\z, h)\le \min_{f\in\HH}q(\z, f) + \frac{2\Delta}{\varepsilon}\cdot\ln(\vert\HH\vert/\beta)
\end{equation*}
with probability at least $1 - \beta$.
\end{lemma}

%% file: text/item.tex
\section{Item-Level Privacy}
\label{sec:item}

In this section, we give a nearly tight characterization of the sample complexity of agnostic learning under item-level privacy.

The work of~\citet{beimel2019characterizing} has shown that the sample complexity of pure private realizable learning is $\widetilde{\Theta}\left(\frac{\repd(\C)}{\alpha\varepsilon}\right)$. Combining with the well-known $\widetilde{\Omega}\left(\frac{\vc(\C)}{\alpha^2}\right)$ lower bound on non-private agnostic learning~\citep{vapnik1974theory,simon1996general}, we get a lower bound of $\widetilde{\Omega}\left(\frac{\repd(\C)}{\alpha\varepsilon}+\frac{\vc(\C)}{\alpha^2}\right)$ on the sample complexity of pure private agnostic learning.

For the upper bound, the first result was given by~\citet{kasiviswanathan2011can}. In their work, they proposed an algorithm using $\widetilde{O}\Big(\frac{\vert\C\vert}{\alpha\varepsilon} + \frac{\vert\C\vert}{\alpha^2}\Big)$ samples for finite concept class $\C$ and $\widetilde{O}\left(\frac{\vc(\C)\log\vert\X\vert}{\alpha\varepsilon} + \frac{\vc(\C)\log\vert\X\vert}{\alpha^2}\right)$ samples for finite domain $\X$. Their analysis can be applied to the realizable learner of~\citet{beimel2019characterizing} by proving the generalization property for the hypothesis class $\HH$ sampled from representation $\PP$. This leads to a sample complexity of $\widetilde{O}\left(\frac{\repd(\C)}{\alpha\varepsilon} + \frac{\repd(\C)}{\alpha^2}\right)$, which exceeds the lower bound since the representation dimension can be much larger than the VC dimension~\citep{feldman2014sample}.

Another existing approach is to use the realizable-to-agnostic transformation~\citep{beimel2014learning,alon2020closure}, which states that every private realizable learner can be transformed into a private agnostic learner. Applying it to the algorithm of~\citet{beimel2019characterizing} gives a sample complexity of $\widetilde{O}\left(\frac{\repd(\C)}{\alpha} + \frac{\vc(\C)}{\alpha^2}\right)$ for constant privacy parameter. For an arbitrary $\varepsilon$, one has to use the amplification-by-subsampling trick~\citep{kasiviswanathan2011can}, resulting in a sample complexity of $\widetilde{O}\left(\frac{\repd(\C)}{\alpha\varepsilon}+\frac{\vc(\C)}{\alpha^2\varepsilon}\right)$. Though here the second term depends on the VC dimension, it involves a $1/\varepsilon$ multiplicative factor, which does not appear in the lower bound. Thus, this method is still suboptimal.

To obtain a tighter result, we propose a new algorithm whose sample complexity matches the lower bound. Before presenting our algorithm, we first introduce some definitions. Let $\z = (z_1, \dots, z_n)$ be the input dataset, where $z_i = (x_i, y_i)\in \X\times\{0, 1\}$. We use $\x = (x_1,\dots, x_n)$ to denote the corresponding unlabeled dataset. For a hypothesis $h$, define
\begin{equation*}
    \err_{\z}(h) = \frac{1}{n}\sum_{i=1}^n \ind[h(x_i)\neq y_i]
\end{equation*}
to be the empirical error of $h$ on $\z$. For two hypotheses $c$ and $h$, define their empirical disagreement on $\x$ as
\begin{equation*}
    \dis_{\x}(c, h) = \frac{1}{n}\sum_{i=1}^n \ind[c(x_i)\neq h(x_i)].
\end{equation*}

Our algorithm follows the steps of the learner proposed by~\citet{beimel2019characterizing}: first samples an $\HH$ from some probabilistic representation $\PP$, then runs the exponential mechanism on $\HH$. But unlike their algorithm, which uses $\err_{\z}(h)$ as the score function, we adopt the following:
\begin{equation*}
    q(\z, h) = \min_{c\in\C}\err_{\z}(c) + \dis_{\x}(c, h).
\end{equation*}

The sensitivity of the above score function is $2/n$ because each term may change by at least $1/n$ when moving to an adjacent dataset. Therefore, we can still apply the exponential mechanism to ensure privacy. The benefit of adopting such a score function is that it reduces the number of samples needed so that every hypothesis in $\HH$ with a low score must have a small generalization error. Proving this for the algorithm of~\citet{beimel2019characterizing} requires agnostic generalization for all $h\in\HH$, resulting in a sample complexity of $\widetilde{O}\left(\frac{\repd(\C)}{\alpha^2}\right)$. When applying our score function, we instead only need agnostic generalization for all $c\in\C$ and realizable generalization for the disagreement between all $c\in\C$ and $h\in\HH$. The former can be ensured by the following agnostic generalization result~\citep{talagrand1994sharper,anthony1999neural}, which requires $\widetilde{O}\left(\frac{\vc(\C)}{\alpha^2}\right)$ samples only.

\begin{lemma}[VC Agnostic Generalization Bound]
\label{lem:vc_ag}
    Let $\C$ be a concept class over $\X$, $\D$ be a distribution over $\Z=\X\times\{0, 1\}$, and 
    \begin{equation*}
        n\ge \frac{576}{\alpha^2}\left(4\vc(\C)+\ln(8/\beta)\right).
    \end{equation*}
    Suppose $\z\in\Z^n$ is a dataset with each $z_i$ drawn i.i.d from $\D$, then
    \begin{equation*}
        \Pr\left[\exists c\in\C~s.t.~\vert\err_{\D}(c) - \err_{\z}(c)\vert>\alpha\right] \le \beta.
    \end{equation*}
\end{lemma}

For the latter one, we can apply the following realizable generalization bound~\citep{vapnik1971uniform,blumer1989learnability} to $\C\cup\HH$. 

\begin{lemma}[VC Realizable Generalization Bound]
\label{lem:vc_re}
    Let $\C$ be a concept class and $\D_{\X}$ be a distribution over domain $\X$. Let 
    \begin{equation*}
        n\ge \frac{96}{\alpha}\left(2\vc(\C)\ln(384/\alpha)+\ln(4/\beta)\right).
    \end{equation*}
    Suppose $\x\in\X^n$ is an unlabeled dataset with each $x_i$ drawn i.i.d. from $\D_\X$, then
    \begin{equation*}
        \Pr\left[
            \exists c, h\in\C~s.t.~\err_{\D_\X}(c, h)>\alpha~
            and~\err_{\x}(c, h) \le \alpha/2
        \right]\le \beta.
    \end{equation*}
\end{lemma}

Note that realizable generalization saves a $1/\alpha$ factor compared to agnostic generalization. This saving is crucial for our tighter upper bound. To identify the VC dimension of $\C\cup\HH$, we need the following bound on the VC dimension of the union.

\begin{lemma}[~\citep{shalev2014understanding}]
\label{lem:union}
    Let $\HH_1$ and $\HH_2$ be two hypothesis class. Then $\vc(\HH_1\cup\HH_2) = O(\vc(\HH_1) + \vc(\HH_2))$.
\end{lemma}

This lemma suggests that the VC dimension of $\C\cup\HH$ is $\widetilde{O}(\repd(\C))$. Thus, the realizable generalization property requires $\widetilde{O}\left(\frac{\repd(\C)}{\alpha}\right)$ samples to hold. Combining with the $\widetilde{O}\left(\frac{\repd(\C)}{\alpha\varepsilon}\right)$ cost incurred by the exponential mechanism, we achieve the desired upper bound.

\begin{theorem}
\label{thm:item}
    Let $\C$ be a concept class. In the item-level setting, $\widetilde{\Theta}\left(\frac{\repd(\C)}{\alpha\varepsilon} + \frac{\vc(C)}{\alpha^2}\right)$ samples are necessary and sufficient to $(\alpha, \beta)$-agnostically learn $\C$ with $\varepsilon$-differential privacy.
\end{theorem}

%% file: text/user.tex
\section{User-Level Privacy}
\label{sec:user}

In this section, we study the user complexity of pure private agnostic learning under user-level privacy. The work of~\citet{ghazi2023user} gives an upper bound of $\widetilde{O}\left(\frac{\repd(\C)}{\varepsilon}+\frac{\repd(\C)}{\sqrt{m}\alpha\varepsilon}+\frac{\repd(\C)}{m\alpha^2}\right)$ and a lower bound of $\widetilde{\Omega}\left(\frac{\repd(\C)}{\varepsilon}+\frac{\vc(\C)}{\sqrt{m}\alpha\varepsilon}+\frac{\vc(\C)}{m\alpha^2}\right)$. We present an algorithm that improves the last term in their upper bound to $\widetilde{O}\left(\frac{\vc(\C)}{m\alpha^2}\right)$, matching that in their lower bound.

Our algorithm works in the same way as the item-level one (Section~\ref{sec:item}), but with the empirical error and disagreement replaced by their user-level analogies. For a concept $c$ and a dataset $\z\in\Z^{nm}$, define the user-level empirical error 
\begin{equation*}
    \err_{\z}^{t}(c) = \frac{1}{n}\sum_{i=1}^{n}\ind[m\cdot\err_{\z_i}(c) > t ]
\end{equation*}
for some $t\in\{0, \dots, m\}$. That is, $\err_{\z}^{t}(c)$ is the fraction of users on whose examples $c$ makes more than $t$ mistakes. Also, we can define the user-level generalization error of $c$ on distribution $\D^m$:
\begin{equation*}
    \err_{\D^m}^t(c) = \Pr_{\z_0\sim\D^m}[m\cdot\err_{\z_0}(c) > t].
\end{equation*}

Similarly, let $\x_i$ be the corresponding unlabeled dataset held by user $i$. For two hypotheses $c$ and $h$, define their user-level empirical disagreement on unlabeled dataset $\x\in\X^{nm}$ to be
\begin{equation*}
    \dis_{\x}^s(c, h) = \frac{1}{n}\sum_{i=1}^{n}\ind[m\cdot\dis_{\x_i}(c, h) > s]
\end{equation*}
and their user-level generalization disagreement on distribution $\D_\X^m$ to be
\begin{equation*}
    \dis_{\D_\X^m}^s(c, h) = \Pr_{\x_0\sim\D_\X^m}[m\cdot\dis_{\x_0}(c, h) > s].
\end{equation*}

The advantage of adopting user-level error is that it amplifies the additive error by a multiplicaive factor of $\sqrt{m}$. This is due to the following property of binomial distribution~\citep{adell2006exact,liu2020learning}.

\begin{lemma}
\label{lem:bin}
    Given $m\in\mathbb{N}$ and $p, q\in[0, 1]$. Let
    \begin{equation*}
        K = \min\left(m|p-q|,\frac{\sqrt{m}|p-q|}{\sqrt{p(1-p)}}, 1\right).
    \end{equation*}
    Then we have
    \begin{equation*}
        \frac{1}{700}K\le d_{TV}(\bin(m, p), \bin(m, q)) \le K,
    \end{equation*}
    where $\bin(m, p)$ is the binomial distribution with $m$ trials and succeed probability $p$, and $d_{TV}$ is the total variation distance. Moreover, there exists an $\ell$ such that
    \begin{equation*}
        \left\vert\Pr[\bin(m, p) > \ell] - \Pr[\bin(m, q) > \ell]\right\vert \ge \frac{1}{700}K.
    \end{equation*}
\end{lemma}

Let $\eta$ be the minimum achievable item-level error. Then the minimum achievable user-level error can be represented by $\psi = \Pr[\bin(m, \eta) > t]$. Consider a hypothesis $h$ whose item-level error is greater than $\eta + \alpha$. Lemma~\ref{lem:bin} suggests that the user-level error of $h$ is at least $\psi + \Omega(\sqrt{m}\alpha)$ for some appropriate $t$. Thus, if we wish the error of our output hypothesis to be at most $\eta + \alpha$, it is sufficient to find a hypothesis whose user-level error is less than $\psi + O(\sqrt{m}\alpha)$.

The above discussion focuses on user-level generalization error, i.e., user-level error on the underlying distribution $\D$. Since our algorithm only has access to the dataset, we have to show generalization properties for user-level error and disagreement as in the item-level setting. We will make use of the following relative uniform convergence lemma~\citep{anthony1999neural}.

\begin{lemma}
\label{lem:relative}
    Let $\C$ be a concept class over $\X$ and $\D$ be a distribution over $\Z = \X \times \{0, 1\}$. Suppose $\z\in\Z^{n}$ is a dataset with each $z_i$ drawn i.i.d. from $\D$, then for $\gamma\in(0, 1)$ and $\xi>0$, we have
    \begin{align*}
        \Pr\left[\exists c\in\C~s.t.~\err_{\D}(c) > (1+\gamma)\err_{\z}(c) + \xi\right]\le 4\Pi_{\C}(2n)\exp\left(\frac{-\gamma\xi n}{4(\gamma+1)}\right).
    \end{align*}
\end{lemma}

The $\Pi_{\C}(2n)$ term in the above lemma is the growth function, which represents the maximum number of labelings of $2n$ samples by some concept in $\C$ (see Appendix~\ref{sec:add} for the exact definition). Note that this lemma naturally extends to the user-level setting: for agnostic generalization, one could pack each user's samples as a whole, which can be regarded as a single data drawn from the distribution $\D^m$ over domain $\Z^m$. Thus, Lemma~\ref{lem:relative} holds with the $\Pi_{\C}(2n)$ term increases to $\Pi_{\C}(2nm)$ since there are $nm$ samples in total. Sauer's Lemma~\citep{sauer1972density} suggests that this only raises the number of users required by roughly $\vc(\C) \ln(m)/\alpha^2$. Standard argument (see, e.g., the book of~\citet{anthony1999neural}) gives the following generalization bound on user-level error.

\begin{lemma}[User-Level Agnostic Generalization]
\label{lem:user_ag}
    Let $\C$ be a concept class over $\X$, $\D$ be a distribution over $\Z = \X\times\{0, 1\}$, and $m$ be the number of samples per user. Let
    \begin{equation*}
        n\ge \frac{64}{\alpha^2}\left(\vc(\C)\ln(128m/\alpha^2)+\ln(8/\beta)\right).
    \end{equation*}
    Suppose $\z\in\Z^{nm}$ is a dataset with each $z_{i, j}$ drawn i.i.d. from $\D$, then
    \begin{equation*}
        \Pr\left[\exists c\in\C~s.t.~\vert\err_{\D^m}^t(c) - \err_{\z}^t(c)\vert > \alpha\right]\le\beta.
    \end{equation*}
\end{lemma}

For the realizable case, we can similarly prove the following bound by setting $\gamma$ and $\xi$ appropriately. Like the item-level setting, the bound is proportional to $1/\alpha$, as opposed to $1/\alpha^2$ in the agnostic case.

\begin{lemma}[User-Level Realizable Generalization]
\label{lem:user_re}
    Let $\C$ be a concept class, $\D_\X$ be a distribution over domain $\X$, and $m$ be the number of samples per user. Let
    \begin{equation*}
        n\ge \frac{96}{\alpha}\left(2\vc(\C)\ln(384m/\alpha)+\ln(4/\beta)\right).
    \end{equation*}
    Suppose $\x\in\X^{nm}$ is an unlabeled dataset with each $x_{i,j}$ drawn i.i.d. from $\D_\X$, then
    \begin{equation*}
        \Pr\left[
            \exists c, h\in\C~s.t.~\dis_{\D_\X^m}^s(c, h)>\alpha~and~\dis_{\x}^s(c, h) \le \alpha/2
        \right]\le \beta.
    \end{equation*}
\end{lemma}

Now, we still have one missing piece: what are the values of $t$ and $s$? It is easy to choose the value of $s$ since we only have to separate $\bin(m, 0)$ and $\bin(m, C\alpha)$ for predetermined constant $C$. Thus, $s$ is a predetermined constant (indeed, $s = 0$). However, for $t$ we have to select a value that separates $\bin(m, \eta)$ and $\bin(m, \eta + O(\alpha))$. Therefore, the choice of $t$ depends on the minimum achievable error $\eta$, which is unknown since $\D$ is unknown.

We resolve this issue by using an approximate value of $\eta$ to decide the value of $t$. It can be shown that an estimation within error $O(\alpha)$ suffices. We design an algorithm that estimates $\eta$ privately. Our idea is based on binary search. Suppose in each iteration, we can compare $\eta$ to the midpoint of the current interval. Then we can obtain an estimation of $\eta$ within error $O(\alpha)$ after $O(\log(1/\alpha))$ iterations.

In each iteration, we have to compare $\eta$ to some guessed value $\widetilde{\eta}$ (the midpoint). Due to agnostic generalization, the comparison can be done by calculating the minimum empirical user-level error with some $\widetilde{t}$, where the parameter $\widetilde{t}$ can be derived from $\widetilde{\eta}$. To ensure privacy, we have to add Laplace noise to the error.

As illustrated before, the user-level error provides a $\sqrt{m}$ amplification. Thus, the binary search yields a reduced user complexity, as stated in the following lemma.

\begin{lemma}
\label{lem:eta}
    Let $\C$ be a concept class and $\D$ be a distribution over $\Z$. Suppose each user holds $m$ samples. Let
    \begin{equation*}
        n\ge \widetilde{O}\left(\frac{1}{\varepsilon} + \vc(\C)\ + \frac{\vc(\C)}{m\alpha^2} + \frac{1}{\sqrt{m}\alpha\varepsilon}\right)
    \end{equation*}
    and $\z\in\Z^{nm}$ be a dataset with each $\z_{i,j}$ drawn i.i.d from $\D$. Then there exists an $\varepsilon$-differentially private algorithm that returns some $\hat{\eta}$ s.t. $\vert\hat{\eta} - \eta\vert\le \alpha$ with probability $1-\beta$, where $\eta = \inf_{c\in\C}\err_{\D}(c)$.
\end{lemma}

Now, we are ready to prove our main result. We first use Lemma~\ref{lem:eta} to compute $\hat{\eta}$, which is an estimation of $\eta$ within error $O(\alpha)$. Then, we construct an appropriate parameter $t$ from $\hat{\eta}$. Finally, we run the exponential mechanism as in the item-level (Theorem~\ref{thm:user}), but with a score function constructed from the user-level empirical error and disagreement. We state the result in the following theorem.

\begin{theorem}
\label{thm:user}
    Let $\C$ be a concept class and $m$ be the number of samples per user. There exists an $\varepsilon$-differentially private algorithm that $(\alpha, \beta)$-agnostically learns $\C$ using $\widetilde{O}\left(\frac{\repd(\C)}{\varepsilon}+\frac{\repd(\C)}{\sqrt{m}\alpha\varepsilon} + \frac{\vc(\C)}{m\alpha^2}\right)$ users.
\end{theorem}

%% file: text/threshold.tex
\section{Learning Thresholds with User-level Privacy}
\label{sec:threshold}

In this section, we focus on learning thresholds with user-level DP and give a nearly tight bound on the user complexity. In the problem of learning thresholds, the domain is $\X = \{1, \dots, |\X|\}$, and the concept class $\C$ is the collection of all thresholds on $\X$. More formally, a threshold function $f_u$ is specified by an element $u\in\{0\}\cup\X$ so that $f_u(x)=1$ if $x > u$ and $f_u(x) = 0$ if $x\le u$ for $x\in\X$. The concept class $\C=\{f_u\mid u\in\{0\}\cup\X\}$.

The class of thresholds on $\X$ has VC dimension $1$ (see e.g.,~\citep{shalev2014understanding}) and representation dimension $\Theta(\log |X|)$~\citep{feldman2014sample}. Thus, directly applying Theorem~\ref{thm:user} gives a user complexity of $\widetilde{O}\left(\frac{\log|X|}{\varepsilon} + \frac{\log |X|}{\sqrt{m}\alpha\varepsilon} + \frac{1}{m\alpha^2}\right)$. As discussed in Section~\ref{sec:user}, the second term is still larger than the lower bound proved by~\citet{ghazi2023user}. We will show how to reduce this term to $\widetilde{O}\left(\frac{1}{\sqrt{m}\alpha\varepsilon}\right)$.

Our method is based on binary search. We start with the entire range $[0, |\X|]$. In each iteration, we pick some point to split the current interval into two parts. We then examine the minimum achievable error of the thresholds in each part and choose the smaller one to continue. This can be done by injecting Laplace noise to the minimum empirical user-level error, which saves the user complexity by a $\sqrt{m}$ factor as demonstrated in Section~\ref{sec:user}.

Suppose we na\"ively select the middle point of the current interval in each iteration, then the binary search has to run for $O(\log |X|)$ rounds. This incurs a user complexity of $\widetilde{O}\left(\frac{\log|X|}{\sqrt{m}\alpha\varepsilon}\right)$ by the basic composition, which is not satisfactory. Thus, we need to reduce the number of iterations.

Our key insight is that, instead of simply picking the middle point, we choose a point in the middle of the distribution (i.e., the median) in each iteration. Though we cannot hope to find the exact median since $\D$ is unknown, a constant approximation is sufficient. Suppose for a given interval $[l, r]$, we can choose some $mid$ such that
\begin{align*}
    \max\left(\Pr_{x\sim\D_\X}[x\in[l,mid-1]],\Pr_{x\sim\D_\X}[x\in[mid+1,r]]\right)\le\theta\cdot\Pr_{x\in\D_\X}[x\in[l,r]]
\end{align*}
for some constant $\theta < 1$. Then we only have to run the binary search for $O(\log(1/\alpha))$ rounds. After that, the disagreement between any two thresholds in the final interval is $O(\alpha)$. Thus, we can output any one in the final interval.

Note that $mid$ should not be included in any part. Otherwise, the above property may be impossible to hold since the distribution can concentrate at a single point. However, we cannot simply ignore $f_{mid}$ since it may be the only threshold that achieves our target error. Thus, in each iteration, we also calculate the error of $f_{mid}$. Our algorithm will terminate and output $f_{mid}$ if it has the smallest error among all thresholds in the current interval. We describe the steps in Algorithm~\ref{alg:threshold}. 

Our method for finding the median requires an observation from Lemma~\ref{lem:bin} by letting $p = c_1\alpha$ and $q = c_2\alpha$ for some constant $c_1$ and $c_2$:
\begin{equation*}
    d_{TV}(\bin(m, c_1\alpha), \bin(m, c_2\alpha)) = \Theta(m\alpha)
\end{equation*}
when $m = O(1/\alpha)$. Thus, it provides an amplification factor of $m$ when $p$ and $q$ are small. Such an observation was also utilized by~\citet{liu2020learning} to derive upper bounds on the problem of learning discrete distributions and by~\citet{ghazi2023user} to characterize the user complexity of user-level pure private realizable learning. We did not elaborate on this in Section~\ref{sec:user} since the $\sqrt{m}$ amplification is sufficient there. However, for median selection, the $\sqrt{m}$ amplification only produces an undesirable user complexity of $\widetilde{O}\left(\frac{\log\vert\X\vert}{\sqrt{m}\alpha\varepsilon}\right)$.

The above observation implies that, for the probability density in $[l, r]$ to be less than $\alpha$, we only need the probability that one user has more than $\ell$ points in $[l, r]$ to be $O(m\alpha)$. By the realizable generalization property, we can use the empirical counterpart as an alternative and apply the exponential mechanism to select an approximate median privately.

\begin{algorithm}[tb]
\caption{$\mathsf{PrivateThreshold}$}
\label{alg:threshold}
\begin{algorithmic}
    \STATE {\bfseries Input:} dataset $\z\in\Z^{nm}$, privacy parameter $\varepsilon$, number of iterations $T$, user-level error parameter $t$, algorithm $\mathsf{PrivateMedian}$, decaying factor $\theta$.
    \STATE $l\gets 0,r\gets \vert\X\vert$
    \STATE $\varepsilon' \gets \varepsilon / (4T)$
    \FOR{$k\gets 1$ {\bfseries to} $T$}
    \IF{$l = r$}
    \STATE \textbf{break}
    \ENDIF
    \STATE $mid \gets \mathsf{PrivateMedian}(\z, \varepsilon', l, r, \theta^{k - 1})$
    \STATE $v_{mid} \gets \err_{\z}^t(f_{mid}) + \lap(1/n\varepsilon')$
    \STATE $v_l \gets \min_{u\in\{l, \dots, mid-1\}} \err_{\z}^t(f_u) + \lap(1/n\varepsilon')$
    \STATE $v_r \gets \min_{u\in\{mid+1, \dots, r\}} \err_{\z}^t(f_u) + \lap(1/n\varepsilon')$
    \STATE \textcolor{gray}{\% The $\min$ operator returns $+\infty$ if the range is empty.}
    \IF{$v_{mid} < \min(v_l, v_r)$}
    \STATE \textbf{return} $f_{mid}$
    \ELSIF{$v_l < v_r$}
    \STATE $r\gets mid - 1$
    \ELSE 
    \STATE $l\gets mid + 1$
    \ENDIF
    \ENDFOR
    \STATE \textbf{return} $f_l$
\end{algorithmic}
\end{algorithm}

\begin{lemma}
\label{lem:median}
    Let $\z\in\Z^{nm}$ be a dataset, where each $\z_{i,j}$ is drawn i.i.d. from some distribution $\D$, and $m$ be the number of samples per user. Suppose $\Pr_{x\sim\D_\X}[x\in[l, r]] \le \alpha$ and 
    \begin{equation*}
        n\ge \widetilde{O}\left(\frac{\log\vert\X\vert}{\varepsilon} + \frac{\log\vert\X\vert}{m\alpha\varepsilon}\right).
    \end{equation*}
    Then there exists an $\varepsilon$-differentially private algorithm that takes $\z, \varepsilon, l, r, \alpha$ as input and output some $u_0$ such that
    \begin{align*}
        \max\left(\Pr_{x\sim\D_\X}[x\in[l, u_0-1]], \Pr_{x\sim\D_\X}[x\in [u_0+1, r]]\right) \le\frac{2}{3}\alpha
    \end{align*}
    with probability at least $1-\beta$.
\end{lemma}

Now, let us summarize the entire algorithm. We first use Lemma~\ref{lem:eta} to determine the value of $t$. Then we run Algorithm~\ref{alg:threshold} with $\mathsf{PrivateMedian}$ being the algorithm in Lemma~\ref{lem:median}. Note that though the $\frac{\log\vert \X\vert}{m\alpha\varepsilon}=\frac{\repd(\C)}{m\alpha\varepsilon}$ term does not explicitly appear in the lower bound presented by~\citet{ghazi2023user}, they also show that this term is necessary even for pure private realizable learning. Thus, we obtain the following nearly tight user complexity bound.

\begin{theorem}
\label{thm:threshold}
    Let $\C$ be the concept class of thresholds over $\X$ and $m$ be the number of samples per user. Then 
    \begin{equation*}
        \widetilde{\Theta}\left(\frac{\log\vert\X\vert}{\varepsilon} + \frac{\log\vert\X\vert}{m\alpha\varepsilon} +\frac{1}{\sqrt{m}\alpha\varepsilon} + \frac{1}{m\alpha^2}\right)
    \end{equation*}
    users are necessary and sufficient to $(\alpha, \beta)$-agnostically learn $\C$ with $\varepsilon$-differential privacy in the user-level setting.
\end{theorem}

%% file: text/conclusion.tex
\section{Conclusion}

This work investigates private agnostic learning under item-level and user-level pure DP. In the item-level setting, we devise an algorithm that achieves the optimal sample complexity up to polylogarithmic factors. In the user-level setting, we propose a generic learner for arbitrary concept classes, with an enhanced user complexity than the best-known result in~\citep{ghazi2023user}. For the specific task of learning thresholds, we develop a near-optimal upper bound and conjecture that such improvement can also be made in the generic case. We leave the problem of closing the gap for general concept classes as future work. Another interesting direction for future work is to investigate the scenario where different users may hold different amounts of data.

%% file: appendix/preliminaries.tex
\section{Additional Preliminaries}
\label{sec:add}

\subsection{The Vapnik-Chervonenkis Dimension}
The Vapnik-Chervonenkis dimension (VC dimension) is a combinatorial measure of concept classes, which characterizes the sample complexity of (non-private) PAC learning and agnostic learning. Consider a concept class $\C$ over domain $\X$. Let $\x\in\X^n$ be an unlabeled dataset of size $n$. The set of all dichotomies on $\x$ that are realized by $\C$ is denoted by $\Pi_{\C}(\x) = \{(c(\x_1), \dots, c(\x_n))\mid c\in\C\}$. The growth function of $\C$ is defined as
\begin{equation*}
    \Pi_{\C}(n) = \max_{x\in\X^n}\vert\Pi_{\C}(\x)\vert.
\end{equation*}
\begin{definition}
    The VC dimension of a concept class $\C$ is defined as
    the largest number $d$ such that $\Pi_{\C}(d) = 2^d$ (or infinity, if the maximum does not exist).
\end{definition}

Though the number of binary vectors of length $n$ is $2^n$, Sauer's Lemma states that the growth function is polynomially bounded.

\begin{lemma}[Sauer's Lemma]
    Let $\C$ be a concept class with VC dimension $d$. Then for $n\ge d$ we have $\Pi_{\C}(n) \le \left(\frac{en}{d}\right)^d$.
\end{lemma}

The following technical inequality from~\citep{anthony1999neural} is useful in deriving sample/user complexity bounds:

\begin{equation}
\label{equ:tec}
    \ln a \le ab + \ln(1/ b) - 1\text{ for all }a, b > 0.
\end{equation}

Applying Sauer's lemma usually gives an inequality of the form $A + B\ln n \le n$. To find an $n$ that satisfies this inequality, let $a = n$ and $b = 1 / 2B$ in~\eqref{equ:tec}, we get
\begin{equation*}
    A + B\ln n \le A + B(n/2B + \ln(2B) - 1) = A + n/2 + B\ln(2B/e).
\end{equation*}
Thus, it suffices to find an $n$ s.t. $A + B\ln(2B/e)\le n / 2$ or, equivalently, $2A + 2B\ln(2B/e) \le n$.

\subsection{Concentration Bounds}

Let $X_1, \dots, X_n$ be independent Bernoulli random variables with probability $p$ of being $1$. The expected value of their sum is $\mathbb{E}[X_1+\dots X_n] = np$. The following inequalities bound the probability that the summation deviates from the its expected value:
\begin{align*}
    &\Pr\left[\sum_{i = 1}^{n}X_i \ge (1 + \delta)np\right]\le \exp(-\delta^2np / (2 + \delta))\text{ for }\delta > 0,\\
    &\Pr\left[\sum_{i = 1}^{n}X_i \le (1 - \delta)np\right]\le \exp(-\delta^2np / 2)\text{ for } 0 < \delta < 1,\\
    &\Pr\left[\left\vert\sum_{i = 1}^{n}X_i - np\right\vert\ge \delta\right]\le 2\exp(-2\delta^2/ n)\text{ for }\delta \ge 0.
\end{align*}

The first two inequalities are knows as the Chernoff bounds~\citep{chernoff1952measure} and the last one is known as Hoeffding's inequality~\citep{hoeffding1963probability}.

%% file: appendix/item.tex
\section{Proof of Theorem~\ref{thm:item}}

\begin{proof}
    The lower bound can be shown by combining the $\widetilde{\Omega}\left(\frac{\repd(\C)}{\alpha\varepsilon}\right)$ lower bound on private realizable learning~\citep{beimel2019characterizing} and the $\widetilde{\Omega}\left(\frac{\vc(\C)}{\alpha^2}\right)$ lower bound on agnostic learning~\citep{vapnik1974theory,simon1996general}.

    Now, let us focus on the upper bound. Our algorithm $\A$ works as follows: first samples $\HH\sim\PP$, where $\PP$ is a $(\alpha/18, \beta/5)$-probabilistic representation of $\C$ with $\mathrm{size}(\PP) = \repd_{\alpha/18, \beta/5}(\C)$, then runs the exponential mechanism on $\HH$ with privacy parameter $\varepsilon$, sensitivity parameter $\Delta=2/n$, and score function $q(\z, h)$ defined 
    as
    \begin{equation*}
        q(\z, h) = \min_{c\in\C}\err_{\z}(c) + \dis_{\x}(c, h).
    \end{equation*}
    Note that each term in the minimization will change by at most $1/n$ when moving to a neighboring dataset. This implies that the sensitivity of $q$ is at most $2/n$. Thus, the privacy guarantee of the exponential mechanism ensures that $\A$ is $\varepsilon$-differentially private.
    
    Now it remains to show that $\A$ is an $(\alpha, \beta)$-agnostic learner. Let $\eta = \inf_{c\in\C}\err_{\D}(c)$, then there exists some $c'\in\C$ s.t. $\err_{\D}(c') \le \eta + \alpha/18$. By Lemma~\ref{lem:vc_ag}, for $n \ge \widetilde{O}\left(\frac{\vc(\C)}{\alpha^2}\right)$, we have $\vert\err_{\D}(c) - \err_{\z}(c)\vert\le \alpha/18$ for all $c\in\C$ with probability $1 - \beta/5$. Thus, 
    \begin{equation*}
        \err_{\z}(c') \le \err_{\D}(c') + \alpha/18 \le \eta + \alpha/9.
    \end{equation*}
    
    Since $\PP$ is an $(\alpha/18, \beta/5)$-probabilistic representation, we know that with probability $1-\beta/5$, there exists some $h'$ s.t. $\dis_{\D_{\X}}(c', h')\le \alpha/18$. Then by the Chernoff bound, for $n\ge\widetilde{O}\left(1/ \alpha\right)$, with probability $1-\beta/5$, we have
    \begin{equation*}
        \dis_{\x}(c', h')\le 2\dis_{\D_\X}(c', h') \le \alpha/9.
    \end{equation*}
    
    Therefore, by Lemma~\ref{lem:exp}, for $n\ge  \widetilde{O}\left(\frac{\repd(\C)}{\varepsilon\alpha}\right)$, the exponential mechanism chooses some $h_0$ s.t. 
    \begin{align*}
        q(\z, h_0) &\le \min_{f\in\HH}q(\z, f) + \alpha/18 \\
        &\le q(\z, h') + \alpha/18 \\
        &\le \err_{\z}(c') + \dis_{\x}(c', h') + \alpha / 18 \\
        &\le \eta + 5\alpha/18
    \end{align*}
    with probability at least $1 - \beta/5$.

    Suppose $q(\z, h_0) = \err_{\z}(c_0) + \dis_{\x}(c_0, h_0)$, we have
    \begin{equation*}
        \err_{\z}(c_0) \le q(\z, h_0) \le \eta + 5\alpha/18.
    \end{equation*}
    Then by agnostic generalization, we can bound the generalization error of $c_0$:
    \begin{equation*}
        \err_{\D}(c_0)\le \err_{\z}(c_0) + \alpha/18\le \eta + \alpha/3.
    \end{equation*}

    Since every concept in $\C$ has error at least $\eta$ on $\D$, agnostic generalization also implies that \begin{equation*}
        \err_{\z}(c_0) \ge \err_{\D}(c_0) - \alpha/18 \ge \eta - \alpha/18.
    \end{equation*}
    Thus, we can bound the empirical disagreement between $c_0$ and $h_0$:
    \begin{align*}
        \dis_{\z}(c_0, h_0) &= q(\z, h_0) - \err_{\z}(c_0)\\
        &\le \eta + 5\alpha/18 - (\eta - \alpha/18) \\
        &=\alpha/3.
    \end{align*}

    We now show the realizable generalization property. By Lemma~\ref{lem:boosting} and~\ref{lem:union}, we have
    \begin{align*}
        \vc(\C\cup\HH) &= O(\vc(\C) + \vc(\HH)) \\
        &= O(\vc(\C) + \log_2\vert\HH\vert) \\
        &= O(\vc(\C) + \repd_{\alpha/18, \beta/5}(\C)) \\
        &= \widetilde{O}(\vc(\C) + \repd(\C)) \\
        &= \widetilde{O}(\repd(\C)).
    \end{align*}

    Then applying Lemma~\ref{lem:vc_re}, for $n \ge \widetilde{O}\left(\frac{\repd(\C)}{\alpha}\right)$, with probability $1-\beta/5$ it holds that $\dis_{\D_\X}(c, h)\le 2\alpha/3$ for all $(c, h)\in\C\times\HH$ s.t. $\dis_{\x}(c, h) \le \alpha/3$. Thus, $\dis_{\D_\X}(c_0, h_0)\le 2\alpha/3$.

    By the union bound, we can conclude that
    \begin{align*}
        \err_{\D}(h_0) &\le \err_{\D}(c_0) + \dis_{\D_\X}(c_0, h_0) \\ &\le \eta + \alpha/3 + 2\alpha/3 \\ &= \eta + \alpha
    \end{align*}
    with probability $1 - \beta$ for some $n = \widetilde{O}\left(\frac{\repd(\C)}{\alpha\varepsilon} + \frac{\vc(\C)}{\alpha^2}\right)$.

\end{proof}

%% file: appendix/user.tex
\section{Proofs from Section~\ref{sec:user}}

\subsection{Proof of Lemma~\ref{lem:user_ag}}
\begin{proof}
    For a concept $c\in\C$, define a concept $\varphi_c: \Z^m\to\{0, 1\}$ such that $\varphi_c(\z_0) = \ind[m\cdot\err_{\z_0}(c) > t]$ for $\z_0\in\Z^m$. Let $\C_{\varphi} = \{\varphi_c\mid c\in\C\}$, $\D'$ be a distribution on $\Z^m\times\{0, 1\}$ s.t.
    \begin{equation*}
        \Pr_{\D'}[(\z_0, y_0)] = \begin{cases}
            \Pr_{\D^m}[\z_0], &y_0=0, \\
            0, &y_0=1,
        \end{cases}
    \end{equation*}
    and $\z'=(z'_1,\dots,z'_n)\in(\Z^m\times\{0, 1\})^n$ be a dataset with $z'_i = (\z_i, 0)$, where $\z=(\z_1,\dots,\z_n)$ is a dataset sampled from $\D^{nm}$. We have
    \begin{equation*}
        \err_{\D'}(\varphi_c) = \Pr_{(\z_0, y_0)\sim\D'}[\varphi_c(\z_0)\neq y_0] =\Pr_{\z_0\sim\D^m}[m\cdot\err_{\z_0}(c)>t] =\err_{\D^m}^t(c)
    \end{equation*}
    and
    \begin{equation*}
        \err_{\z'}(\varphi_c) = \frac{1}{n}\sum_{i=1}^n\ind[\varphi_c(\z_i)\neq 0] = \frac{1}{n}\sum_{i=1}^n\ind[m\cdot\err_{\z_i}(c)>t]= \err_{\z}^t(c).
    \end{equation*}
    Moreover, each $z'_i$ can be seen as i.i.d. drawn from $D'$. Applying Lemma~\ref{lem:relative} with concept class $\C_{\varphi}$, distribution $\D'$, dataset $\z'$, and $\gamma = \xi = \alpha/2$, we have
    \begin{align*}
        &{}\Pr[\exists c\in\C~s.t.~\err_{\D^m}^t(c) > \err_{\z}^t(c) + \alpha] \\
        =&\Pr[\exists \varphi_c\in\C_{\varphi}~s.t.~\err_{\D'}(\varphi_c) > \err_{\z'}(\varphi_c)+\alpha] \\
        \le&\Pr[\exists \varphi_c\in\C_{\varphi}~s.t.~\err_{\D'}(\varphi_c) > (1+\alpha/2)\err_{\z'}(\varphi_c)+\alpha/2] \\
        \le&4\Pi_{\C_{\varphi}}(2n)\exp\left(\frac{-\alpha^2n/4}{4(1+\alpha/2)}\right) \\
        \le&4\Pi_{\C}(2nm)\exp\left(-\alpha^2n/32\right),
    \end{align*}
    where the last inequality is due to the fact that
    \begin{align*}
        \Pi_{\C_{\varphi}}(2n) &= \max_{\z\in\Z^{2nm}}\vert\Pi_{\C_\varphi}(\z)\vert \\ 
        &= \max_{\z\in\Z^{2nm}} \left\vert\left\{\left(\ind\left[\sum_{i=1}^m\ind\left[c(x_{1,i})\neq y_{1,i}\right]>t\right],\dots,\ind\left[\sum_{i=1}^m\ind\left[c(x_{2n,i})\neq y_{2n,i}\right]>t\right]\right)~\middle\vert~c\in\C\right\}\right\vert\\
        &\le\max_{\z\in\Z^{2nm}}\left\vert\left\{\left(\ind[c(x_{1, 1})\neq y_{1, 1}], \dots, \ind[c(x_{2n, m})\neq y_{2n, m}]\right)\mid c\in\C\right\}\right\vert\\
        &=\max_{\x\in\X^{2nm}}\vert\Pi_{\C}(\x)\vert \\
        &= \Pi_{\C}(2nm).
    \end{align*}
    By Sauer's Lemma, we have $\Pi_{\C}(2nm) \le \left(\frac{2enm}{d}\right)^d$ if $2nm\ge d$, where $d = \vc(\C)$. We want the above quantity to be at most $\beta/2$. Thus, it suffices to show
    \begin{equation*}
        n\ge \frac{32}{\alpha^2}\left(d\ln n + d\ln(2em/d)+\ln(8/\beta)\right).
    \end{equation*}
    By inequality~\eqref{equ:tec}, we have
    \begin{equation*}
        \frac{32d}{\alpha^2}\ln n \le \frac{n}{2} + \frac{32d}{\alpha^2}\ln\left(\frac{64d}{e\alpha^2}\right).
    \end{equation*}
    Therefore,
    \begin{equation}
    \label{equ:user_ag}
        n\ge \frac{64}{\alpha^2}\left(d\ln(128m/\alpha^2)+\ln(8/\beta)\right)
    \end{equation}
    suffices.
    We still need to prove the other side. For a concept $c$, define $c^{-}(x) = 1 - c(x)$. Let $\C^- = \{c^-\mid c\in\C\}$ and $t^- = m - t - 1$. It is easy to verify that $\vc(\C^-) = \vc(\C)$, $\err_{\z}^{t^-}(c^-) = 1 - \err_{\z}^t(c)$, and $\err_{\D^m}^{t^-}(c^-) = 1 - \err_{\D^m}^t(c)$. Applying the above conclusion to $\C^-$ and $t^-$ gives
    \begin{align*}
        {}&\Pr[\exists c\in\C~s.t.~\err_{\D^m}^t(c) + \alpha < \err_{\z}^t(c)] \\
        =& \Pr[\exists c^-\in\C^-~s.t.~1 - \err_{\D^m}^{t^-}(c^-) + \alpha < 1 - \err_{\z}^{t^-}(c^-)] \\
        =& \Pr[\exists c^-\in\C^-~s.t.~\err_{\D^m}^{t^-}(c^-) > \err_{\z}^{t^-}(c^-) + \alpha] \\
        \le& \beta/2
    \end{align*}
    if $n$ satisfies inequality~\eqref{equ:user_ag}. Applying the union bound completes the proof.
\end{proof}

\subsection{Proof of Lemma~\ref{lem:user_re}}

\begin{proof}
    For any $c, h\in\C$, define a concept $\varphi_{c,h}:\X^m\to\{0, 1\}$ such that $\varphi_{c, h}(\x_0) = \ind[m\cdot\dis_{\x_0}(c, h) > s]$ for $\x_0\in\X^m$. Let $\C_{\varphi} = \{\varphi_{c,h}\mid c,h\in\C\}$, $\D'$ be a distribution on $\X^m\times\{0, 1\}$ s.t.
    \begin{equation*}
        \Pr_{\D'}[(\x_0,y_0)] = \begin{cases}
            \Pr_{\D_\X^m}[\x_0], &y_0=0, \\
            0, &y_0=1,
        \end{cases}
    \end{equation*}
    and $\z' \in (\X^m\times\{0, 1\})^n$ be a dataset with $z'_i = (\x_i, 0)$. Similar to the proof of Lemma~\ref{lem:user_ag}, we have $\err_{\D'}(\varphi_{c, h}) = \dis_{\D_\X^m}^s(c, h)$ and $\err_{\z'}(\varphi_{c, h}) = \dis_{\x}^s(c, h)$. By Lemma~\ref{lem:relative}, let $\gamma=1/2$ and $\xi = \alpha/4$, we have
    \begin{align*}
        {}&\Pr\left[\exists c, h\in\C~s.t.~\dis_{\D_\X^m}^s(c, h)>\alpha~and~\dis_{\x}^s(c, h) \le \alpha/2 \right] \\
        \le&\Pr\left[\exists c, h\in\C~s.t.~\dis_{\D_{\X}^m}^s(c, h) > 3/2\dis_{\x}^s(c, h) + \alpha/4\right] \\
        =& \Pr[\exists\varphi_{c, h}\in\C_{\varphi}~s.t.~\err_{\D'}(\varphi_{c, h}) > 3/2\err_{\z'}(\varphi_{c, h}) + \alpha/4] \\
        \le&4\Pi_{\C_{\varphi}}(2n)\exp\left(-\frac{-1/2\cdot\alpha n/4}{4\cdot(1 + 1/2)}\right) \\
        \le&4\left(\Pi_{\C}(2nm)\right)^2\exp(-\alpha n/48),
    \end{align*}
    where the last inequality comes from the fact that
    \begin{align*}
        &\Pi_{\C_{\varphi}}(2n) \\=& \max_{\x\in\X^{2nm}}\vert\Pi_{\C_\varphi}(\x)\vert \\
        =& \max_{\x\in\X^{2nm}}\left\vert\left\{\left(\ind\left[\sum_{i=1}^m\ind[c(x_{1,i})\neq h(x_{1, i})] > s\right], \dots,\ind\left[\sum_{i=1}^m\ind[c(x_{2n,i})\neq h(x_{2n, i})] > s\right]\right)~\middle|~ c, h\in\C\right\}\right\vert \\
        \le& \max_{\x\in\X^{2nm}}\left\vert\{(\ind[c(x_{1,1})\neq h(x_{1, 1})], \dots,\ind[c(x_{2n,m})\neq h(x_{2n, m})])\mid c, h\in\C\}\right\vert \\
        \le& \max_{\x\in\X^{2nm}} \vert\Pi_{\C}(\x)\vert\cdot\vert\Pi_{\C}(\x)\vert \\
        =&(\Pi_{\C}(2nm))^2.
    \end{align*}
    Let $d = \vc(\C)$, Sauer's Lemma ensures that the above quantity is at most $\beta$ if $2nm\ge d$ and
    \begin{equation*}
        n\ge \frac{48}{\alpha}\left(2d\ln n + 2d\ln(2em/d)+ \ln(4/\beta)\right).
    \end{equation*}
    By inequality~\eqref{equ:tec}, we have
    \begin{equation*}
        \frac{96d}{\alpha}\ln n \le \frac{n}{2} + \frac{96d}{\alpha}\ln\left(\frac{192d}{e\alpha}\right)
    \end{equation*}
    Thus, 
    \begin{equation*}
        n\ge \frac{96}{\alpha}\left(2d\ln(384m/\alpha)+\ln(4/\beta)\right)
    \end{equation*}
    suffices.
\end{proof}

\subsection{Proof of Lemma~\ref{lem:eta}}

The key insight behind Lemma~\ref{lem:eta} is that we can decide whether a given value $\widetilde{\eta}$ is greater or lees than the minimum error $\eta$ up to an error of $O(\alpha)$, as demonstrated in the following lemma.

\begin{lemma}
\label{lem:comp}
    Let $\C$ be a concept class, $\D$ be a distribution over $\Z$, $m$ be the number of samples per user, and $\eta = \inf_{c\in\C}\err_{\D}(c)$. Suppose 
    \begin{equation*}
        n\ge \widetilde{O}\left(\frac{1}{\varepsilon} + \vc(\C) + \frac{\vc(\C)}{m\alpha^2} + \frac{1}{\sqrt{m}\alpha\varepsilon}\right)
    \end{equation*}
    and $\z\in\Z^{nm}$ be a dataset with each $\z_{i, j}$ drawn i.i.d. from $\D$. Then there exists an $\varepsilon$-differentially private algorithm that takes $\z,\varepsilon, \widetilde{\eta}, \alpha$ as input and outputs some $\sigma\in\{0, 1\}$, such that with probability $1-\beta$:
    \begin{enumerate}
        \item If $\sigma = 0$, then $\widetilde{\eta} \ge \eta - \alpha/2$.
        \item If $\sigma = 1$, then $\widetilde{\eta} \le \eta + \alpha/2$.
    \end{enumerate}
\end{lemma}

\begin{proof}
    We will assume $m \le 1/\alpha^2$ and prove an upper bound of $\widetilde{O}\left(\frac{\vc(\C)}{m\alpha^2} + \frac{1}{\sqrt{m}\alpha\varepsilon}\right)$. When $m > 1/\alpha^2$, the result simply holds by discarding the extra samples.

    By Lemma~\ref{lem:bin}, there exists some $t$ s.t.
    \begin{equation*}
        \Pr[\bin(m, \widetilde{\eta}) > t] + \sqrt{m}\alpha / 1400 \le \Pr[\bin(m, \widetilde{\eta} + \alpha / 2) > t].
    \end{equation*}
    Let $\A$ be an algorithm that returns $0$ if $\min_{c\in\C}\err_{\z}^t(c) + r\le  \Pr[\bin(m, \widetilde{\eta}) > t] + \sqrt{m}\alpha / 2800$, where $r\sim\lap(1/\varepsilon n)$, and otherwise returns $1$. By Lemma~\ref{lem:lap}, $\A$ is $\varepsilon$-differentially private and $\vert r\vert\le \sqrt{m}\alpha / 5600$ with probability $1-\beta / 2$ if $n\ge\widetilde{O}(\frac{1}{\sqrt{m}\alpha\varepsilon})$.

    The agnostic generalization bound (Lemma~\ref{lem:user_ag}) shows that if $n\ge\widetilde{O}\left(\frac{\vc(\C)}{\sqrt{m}\alpha^2}\right)$, then with probability $1-\beta / 2$ it holds that $\vert\err_{\z}^t(c) - \err_{\D^m}^t(c)\vert \le \sqrt{m}\alpha/5600$ for all $c\in\C$. For the remainder of the proof, we condition on the event that $\vert r\vert\le \sqrt{m}\alpha / 5600$ and $\vert\err_{\z}^t(c) - \err_{\D^m}^t(c)\vert \le \sqrt{m}\alpha/5600$ for all $c\in\C$. By the union bound, this happens with probability $1-\beta$.

    First consider the case that $\eta \le \widetilde{\eta}$. We have
    \begin{align*}
        \min_{c\in\C}\err_{\z}^t(c) + r& \le \Pr[\bin(m, \eta) > t] + \sqrt{m}\alpha/5600 + r\\
        &\le \Pr[\bin(m, \widetilde{\eta}) > t] + \sqrt{m}\alpha / 5600 + \sqrt{m}\alpha / 5600 \\&\le \Pr[\bin(m, \widetilde{\eta}) > t] + \sqrt{m} \alpha/ 2800.
    \end{align*}
    Thus, $\A$ must return $0$ in this case. Equivalently, it guarantees that $\eta > \widetilde{\eta} \ge \widetilde{\eta} -\alpha/2$ if $\A$ returns $1$.

    Now suppose $\eta > \widetilde{\eta}$. Note that if $\eta > \widetilde{\eta} + \alpha / 2$, it holds that
    \begin{equation*}
        \Pr[\bin(m, \widetilde{\eta}) > t] + \sqrt{m} /1400 < \Pr[\bin(m, \eta) > t] = \inf_{c\in\C}\err_{\D^m}^t(c).
    \end{equation*}
    Thus, if $\inf_{c\in\C}\err_{\D^m}^t(c) \le \Pr[\bin(m, \widetilde{\eta} )>t] + \sqrt{m} /1400$ we must have $\eta \le \widetilde{\eta} + \alpha / 2$.
    
    Suppose $\A$ returns $0$. Let $c_0\in\C$ s.t. $\err_{\z}^t(c_0) = \min_{c\in\C}\err_{\z}^t(c)$. Then we have
    \begin{align*}
        \err_{\D^m}^t(c_0) &\le \err_{\z}^t(c_0) + \sqrt{m}\alpha / 5600\\&\le \Pr[\bin(m, \widetilde{\eta}) > t] + \sqrt{m}\alpha / 2800 -r+ \sqrt{m}\alpha / 5600 \\&\le \Pr[\bin(m,\widetilde{\eta})>t]+\sqrt{m}\alpha/1400.
    \end{align*}
    Therefore, $\inf_{c\in\C}\err_{\D^m}^t(c) \le \Pr[\bin(m, \widetilde{\eta}) > t] + \sqrt{m}\alpha / 1400$, which implies that $\eta \le \widetilde{\eta} + \alpha/2$.

    In summary, with probability $1-\beta$, we have $\widetilde{\eta} \ge \eta - \alpha/2$ if $\A$ returns $0$, and $\widetilde{\eta}\le \eta + \alpha / 2$ if $\A$ return $1$.
\end{proof}

Now we can accurately estimate the value of $\eta$ by binary research. The details are described in Algorithm~\ref{alg:eta}.

\begin{algorithm}[tb]
\caption{$\mathsf{PrivateMinError}$}
\label{alg:eta}
\begin{algorithmic}
    \STATE {\bfseries Input:} dataset $\z\in\Z^{nm}$, privacy parameter $\varepsilon$, accuracy parameter $\alpha$, confidence parameter $\beta$, algorithm $\mathsf{PrivateCompare}$.
    \STATE $l\gets 0,r\gets 1$
    \STATE $T\gets \lceil\log_2(2/\alpha)\rceil$
    \STATE $\varepsilon' \gets \varepsilon / T, \beta' \gets \beta / T$
    \FOR{$k\gets 1$ {\bfseries to} $T$}
    \STATE $mid \gets (l + r) / 2$
    \IF{$\mathsf{PrivateCompare}(\z, \varepsilon', mid, \alpha, \beta') = 1$}
    \STATE $r\gets mid$
    \ELSE 
    \STATE $l\gets mid$
    \ENDIF
    \ENDFOR
    \STATE {\bfseries return} $l$
\end{algorithmic}
\end{algorithm}

\begin{proof}[Proof of Lemma~\ref{lem:eta}]
    We execute Algorithm~\ref{alg:eta} with $\mathsf{PrivateCompare}$ being the algorithm in Lemma~\ref{lem:comp}.

    Let $l_k, r_k$ be the endpoints after the $k$-th iteration. Initially, we have $l_0 = 0$ and $r_0 = 1$. Then after $T$ iterations, the length of $[l_T, r_T]$ is $r_T - l_T \le 1 / 2^T \le \alpha / 2$.

    We now show that if every call of $\mathsf{PrivateCompare}$ succeeds, then $[l_k, r_k]\cap[\eta - \alpha / 2, \eta + \alpha / 2]\neq\emptyset$ for all $k$. We prove this by induction. At the beginning, we have $[l_0, r_0]\cap[\eta - \alpha/2,\eta + \alpha/2]\neq \emptyset$ since $\eta\in[0, 1]=[l_0,r_0]$. Suppose $[l_{k-1}, r_{k-1}]\cap[\eta - \alpha / 2, \eta+\alpha/ 2]\neq\emptyset$. In the $k$-th iteration, we set $mid = (l_{k - 1} + r_{k-1})/2$. Let $\sigma$ be the return value of $\mathsf{PrivateCompare}$. By Lemma~\ref{lem:comp}, with probability $1 - \beta' = 1 - \beta/T$, we have:
    \begin{enumerate}
        \item If $\sigma = 0$, then $mid \ge \eta - \alpha / 2$. Thus, $[l_{k-1},mid]\cap[\eta - \alpha/2, \eta + \alpha / 2]\neq \emptyset$ since $\eta + \alpha/2 \ge l_{k - 1}$. 
        \item If $\sigma = 1$, then $mid \le \eta + \alpha / 2$. Thus, $[mid, r_{k-1}]\cap[\eta - \alpha / 2, \eta + \alpha / 2] \neq\emptyset$ since $\eta - \alpha /2 \ge r_{k - 1}$.
    \end{enumerate}
    For both cases, we have $[l_{k}, r_{k}]\cap[\eta - \alpha / 2, \eta+\alpha/ 2]\neq\emptyset$.

    By the union bound, it holds with probability $1-\beta$ that $\vert\eta - x\vert \le \alpha / 2$ for some $x\in[l_T, r_T]$. This implies $\vert l_T - \eta\vert \le \vert l_T - x\vert + \vert x - \eta\vert\le \alpha$, which completes the proof.
\end{proof}

\subsection{Proof of Theorem~\ref{thm:user}}

\begin{proof}
    We will assume that $m\le 1/\alpha^2$ and show an upper bound of $\widetilde{O}\left(\frac{\repd(\C)}{\sqrt{m}\alpha\varepsilon} + \frac{\vc(\C)}{m\alpha^2}\right)$. When $m > 1/\alpha^2$, the $\widetilde{O}\left(\frac{\repd(\C)}{\varepsilon}\right)$ term becomes the dominated term and the user complexity can be achieved by discarding the extra samples.

    Let $\eta = \inf_{c\in\C}\err_{\D}(c)$. We first run the algorithm in Lemma~\ref{lem:eta} with privacy parameter $\varepsilon/2$. For $n \ge \widetilde{O}\left(\frac{\vc(\C)}{m\alpha^2} + \frac{1}{\sqrt{m}\alpha\varepsilon}\right)$, it returns some $\hat{\eta}$ s.t. $|\hat{\eta} - \eta|\le \alpha/6$ with probability $1-\beta / 7$.
    
    By Lemma~\ref{lem:bin}, there exists some $t$ s.t.
    \begin{equation*}
        \Pr[\bin(m, \hat{\eta} + \alpha/6) > t] + \sqrt{m}\alpha / 4200 \le \Pr[\bin(m, \hat{\eta} + \alpha/3) > t].
    \end{equation*}

    Let $\psi = \Pr[\bin(m, \eta) > t]$ and $\hat{\psi} = \Pr[\bin(m, \hat{\eta} + \alpha / 6) > t]$. Then we have $\psi\le \hat{\psi}$ since $\eta \le \hat{\eta} + \alpha/6$.
    
    For any concept $c$ with $\err_{\D}(c) > \eta + \alpha / 2$, we have $\err_{\D}(c) > \hat{\eta} + \alpha/3$ by the triangle inequality. Thus,
    \begin{equation*}
        \err_{\D^m}^t(c) > \Pr[\bin(m, \hat{\eta} + \alpha/3) > t] \ge \hat{\psi} + \sqrt{m}\alpha / 4200.
    \end{equation*}
    Thus, for any concept $c$ with $\err_{\D^m}^t(c)\le \hat{\psi} + \sqrt{m}\alpha / 4200$ it must hold that $\err_{\D}(c) \le \eta + \alpha / 2$.

    Again by Lemma~\ref{lem:bin}, there exists some $s$ (actually, $s = 0$) s.t.
    \begin{equation*}
        \Pr[\bin(m, \alpha / 3) > s] \ge \Pr[\bin(m, 0) > s] + \sqrt{m}\alpha / 2100 = \sqrt{m}\alpha / 2100.
    \end{equation*}
    Thus, for any hypotheses $c$ and $h$ with $\dis_{\D_\X^m}^s(c, h) \le \sqrt{m}\alpha / 2100$, it holds that $\dis_{\D_\X}(c, h) \le \alpha / 3$.
    
    Our algorithm then works as follows: first samples $\HH\sim\PP$, where $\PP$ is a $(\alpha / 134400\sqrt{m},\beta / 7)$-probabilistic representation of $\C$ with $\mathrm{size}(\PP) = \repd_{\alpha/134400\sqrt{m},\beta/7}(\C)$, then run the exponential mechanism on $\HH$ with parameter $\varepsilon/2$, sensitivity parameter $\Delta = 2/n$, and score function
    \begin{equation*}
        q(\z, h) = \min_{c\in\C} \err_{\z}^t(c) + \dis_{\x}^s(c, h).
    \end{equation*}
    The privacy is satisfied by the composition property of DP. The proof of the utility guarantee is similar to that of Theorem~\ref{thm:item}. Firstly, there exists some $c'\in\C$ s.t. $\err_{\D^m}^t(c') \le \psi + \sqrt{m}\alpha / 134400$. By Hoeffding's inequality, for $n\ge\widetilde{O}\left(\frac{1}{m\alpha^2}\right)$, we have $\err_{\z}^t(c') \le \err_{\D^m}^t(c') + \sqrt{m}\alpha/ 134400 \le \psi + \sqrt{m}\alpha / 67200$ with probability $1 - \beta / 7$.

    Since $\PP$ is an $(\alpha/134400\sqrt{m}, \beta / 7)$-probabilistic representation, with probability $1 - \beta / 7$ there exists some $h'\in\HH$ s.t. $\dis_{\D_\X}(c', h') \le \alpha / 134400\sqrt{m}$. By the upper bound in Lemma~\ref{lem:bin}, $\dis_{\D_\X^m}^s(c', h') \le \sqrt{m}\alpha / 134400$. Then by the Chernoff bound, for $n\ge \widetilde{O}\left(\frac{1}{\sqrt{m}\alpha}\right)$ it holds that $\dis_{\x}^s(c', h') \le 2\dis_{\D_\X^m}^s(c', h') \le \sqrt{m}\alpha / 67200$ with probability at least $1 - \beta / 7$.
    
    So, $q(\z, h') \le \err_{\z}^t(c') + \dis_{\x}^s(c', h') \le \psi + \sqrt{m}\alpha / 33600$. By Lemma~\ref{lem:exp}, for $n\ge \widetilde{O}\left(\frac{\repd(\C)}{\sqrt{m}\alpha\varepsilon}\right)$, the exponential mechanism chooses some $h_0$ s.t. 
    \begin{equation*}
        q(\z, h_0)\le \min_{f\in\HH}q(\z, f) + \sqrt{m}\alpha / 16800 \le q(\z, h') + \sqrt{m}\alpha / 16800\le \psi + \sqrt{m}\alpha / 8400
    \end{equation*}
    with probability $1-\beta / 7$.

    Let $q(\z, h_0) = \err_{\z}^t(c_0) + \dis_{\x}^s(c_0, h_0)$. Then we have $\err_{\z}^t(c_0) \le q(\z, h_0) \le \psi + \sqrt{m}\alpha / 8400$. By Lemma~\ref{lem:user_ag}, for $n\ge\widetilde{O}\left(\frac{\vc(\C)}{m\alpha^2}\right)$ we have $\vert\err_{\D^m}^t(c) - \err_{\z}^t(c) \vert \le \sqrt{m}\alpha / 8400$ for all $c\in\C$ with probability $1-\beta / 7$. Thus, 
    \begin{equation*}
        \err_{\D^m}^t(c_0) \le \err_{\z}^t(c_0) + \sqrt{m}\alpha / 8400\le \psi + \sqrt{m}\alpha / 4200\le \hat{\psi} +\sqrt{m}\alpha / 4200
    \end{equation*}
    and $\err_{\z}^t(c_0) \ge \psi - \sqrt{m}\alpha / 8400$. The former implies $\err_{\D}(c_0)\le \eta + \alpha / 2$ and the latter implies 
    \begin{equation*}
        \dis_{\x}^s(c_0, h_0) = q(\z, h_0) - \err_{\z}(c_0) \le \psi + \sqrt{m}\alpha / 8400 - (\psi - \sqrt{m}\alpha / 8400) \le \sqrt{m}\alpha / 4200.
    \end{equation*}

    As shown in the proof of Theorem~\ref{thm:item}, the VC dimension of $\C\cup\HH$ is $\widetilde{O}\left(\repd(\C)\right)$. Then by Lemma~\ref{lem:user_re}, for $n\ge\widetilde{O}\left(\frac{\repd(\C)}{\sqrt{m}\alpha}\right)$, it holds with probability $1-\beta /7 $ that $\dis_{\D_\X^m}^s(c, h)\le \sqrt{m}\alpha / 2100$ for all $c\in\C$ and $h\in\HH$ with $\dis_{\x}^s(c, h)\le \sqrt{m}\alpha / 4200$. Thus, we have $\dis_{\D_\X^m}^s(c_0, h_0)\le \sqrt{m}\alpha / 2100$, which further indicates $\dis_{\D_\X}(c_0, h_0) \le \alpha / 3$.

    By the union bound, we have
    \begin{equation*}
        \err_{\D}(h_0) \le \err_{\D}(c_0) + \dis_{\D_\X}(c_0, h_0) \le \eta + \alpha / 2+ \alpha / 3 < \eta + \alpha
    \end{equation*}
    with probability $1-\beta$ for $n\ge \widetilde{O}\left(\frac{\repd(\C)}{\sqrt{m}\alpha\varepsilon} + \frac{\vc(\C)}{m\alpha^2}\right)$.
\end{proof}

%% file: appendix/threshold.tex
\section{Proofs from Section~\ref{sec:threshold}}

\subsection{Proof of Lemma~\ref{lem:median}}

We first prove an corollary of Lemma~\ref{lem:bin} by taking $p=\alpha/2$ and $q = 2\alpha/3$, which states that the total variation distance between $\bin(m, p)$ and $\bin(m, q)$ scales linearly with $m$.

\begin{corollary}
\label{cor:m}
    Suppose $0<\alpha \le 1$ and $m \le 1/\alpha$. Then there exists an $\ell$ such that
    \begin{equation*}
        \Pr[\bin(m, \alpha / 2) > \ell] + m\alpha / 4200\le \Pr[\bin(m, 2\alpha /3) > \ell].
    \end{equation*}
\end{corollary}
\begin{proof}
    Let $p = \alpha / 2$ and $q = 2\alpha / 3$. Then by Lemma~\ref{lem:bin}, there exists an $\ell$ s.t.
    \begin{equation*}
        \Pr[\bin(m, \alpha / 2) > \ell] + K / 700 \le \Pr[\bin(m, 2\alpha /3) > \ell],
    \end{equation*}
    where
    \begin{equation*}
        K = \min\left(m\vert p - q\vert, \frac{\sqrt{m}\vert p - q\vert}{\sqrt{p(1-  p)}}, 1\right) =\min(m\alpha / 6, \sqrt{m / (p(1- p))}\cdot \alpha /6) = m\alpha / 6.
    \end{equation*}
    since $m / (p(1 - p)) > m / p =  2m / \alpha \ge 2m^2$.
\end{proof}

\begin{proof}[Proof of Lemma~\ref{lem:median}]
    We will assume $m \le 1/\alpha$ and prove an upper bound of $\widetilde{O}\left(\frac{\log\vert\X\vert}{m\alpha\varepsilon}\right)$. When $m > 1/\alpha$, the conclusion follows by discarding the extra samples.

    Let $\x$ be the corresponding unlabeled dataset of $\z$. By Corollary~\ref{cor:m}, there exists some $s$ .s.t.
    \begin{equation*}
        \Pr[\bin(m, \alpha / 2) > s] + m\alpha/4200 \le \Pr[\bin(m, 2\alpha/3) > s].
    \end{equation*}

    Let $\dis_{\x}^s(f_{a-1}, f_b)$ be the fraction of users that have more than $ms$ samples lie in $[a, b]$. Define score function
    \begin{equation*}
        q(\z, u) = \max\left(\dis_{\x}^s(f_{l-1}, f_{u-1}),\dis_{\x}^s(f_{u}, f_{r})\right).
    \end{equation*}
    The sensitivity of $q(\z, u)$ is $1/n$. Thus, we can run the exponential mechanism on $\{l, \dots, r\}$ with privacy parameter $\varepsilon$ and $\Delta=1/n$. Lemma~\ref{lem:exp} proves the privacy guarantee. To analyze the accuracy, we start with the claim that there exists some $u'\in\{l,\dots, r\}$ such that
    \begin{equation*}
        \max\left(\Pr_{x\sim\D_\X}[x\in[l, u'-1]], \Pr_{x\sim\D_\X}[x\in [u'+1, r]]\right) \le \alpha/2.
    \end{equation*}
    To show this, let $u'$ be the greatest number in $\{l, \dots, r\}$ s.t. $\Pr_{x\sim\D_\X}[x\in[l, u' - 1]] \le \alpha/2$. We will prove that $\Pr_{x\sim\D_\X}[x\in[u' + 1, r]] \le \alpha/2$. Suppose this does not hold, namely,  $\Pr_{x\sim\D_\X}[x\in[u' + 1, r]] > \alpha/2$. Then by the definition of $u'$ we have $\Pr_{x\sim\D_\X}[x\in[l, u']] > \alpha/2$. Thus, 
    \begin{equation*}
        \alpha\ge\Pr_{x\sim\D_\X}[x\in[l, r]] = \Pr_{x\sim\D_\X}[x\in[l, u']] + \Pr_{x\sim\D_\X}[x\in[u' + 1, r]] > \alpha,
    \end{equation*}
    a contradiction.
    
    Thus, we have
    \begin{equation*}
        \max\left(\Pr_{\x_0\sim\D_\X^m}[\dis_{\x_0}(f_{l-1}, f_{u'-1}) > s], \Pr_{\x_0\sim\D_\X^m}[\dis_{\x_0}(f_{u'}, f_{r}) > s]\right) \le \Pr[\bin(m, \alpha / 2) > s].
    \end{equation*}

    Moreover, by the upper bound in Lemma~\ref{lem:bin}, the right-hand side of the above is at most $m\alpha / 2$. Now apply the Chernoff bound, for $n\ge\widetilde{O}\left(\frac{1}{m\alpha}\right)$, with probability $1 - \beta /3$, we have
    \begin{align*}
        q(\z, u') &= \max\left(\Pr[\dis_{\x}^s(f_{l-1}, f_{u'-1})], \Pr[\dis_{\x}^s(f_{u'}, f_{r})]\right) \\&\le  \Pr[\bin(m, \alpha / 2) > s](1 + 1 / 6300) \\ &\le  \Pr[\bin(m, \alpha / 2) > s] + m\alpha / 12600.
    \end{align*}

    By Lemma~\ref{lem:exp}, for $n\ge \widetilde{O}\left(\frac{\log\vert \X\vert}{m\alpha\varepsilon}\right)$, the exponential mechanism returns a $mid$ s.t. 
    \begin{equation*}
        q(\z, mid) \le \min_{u\in\{l, \dots, r\}}q(\z, u) + m\alpha / 12600 \le q(\z, u') + m\alpha/ 12600 \le   \Pr[\bin(m, \alpha / 2) > s] + m\alpha / 6300
    \end{equation*}
    with probability at least $1-\beta/3$. By the Chernoff bound, when $n\ge\widetilde{O}\left(\frac{\log\vert X\vert}{m\alpha}\right)$, it holds with probability $1 - \beta / 3$ that
    \begin{equation*}
        \dis_{\D_\X^m}^s(f_{u}, f_{v}) \le (\Pr[\bin(m, \alpha / 2) > s] + m\alpha / 6300)\cdot\frac{6303}{6302} 
    \end{equation*}
    for all $u, v\in\{0\}\cup\X$ with $\dis_{\x}^s(f_u, f_v)\le \Pr[\bin(m, \alpha / 2) > s] + m\alpha /6300$. Therefore, we have
    \begin{align*}
        &\max\left(\dis_{\D_\X^m}^s(f_{l-1}, f_{mid-1}),\dis_{\D_\X^m}^s(f_{mid+1}, r)\right) \\\le&(\Pr[\bin(m, \alpha / 2) > s] + m\alpha / 6300)\cdot\frac{6303}{6302} \\\le&\Pr[\bin(m, \alpha / 2) > s] + m\alpha / 6300 + (m\alpha / 2 + m\alpha / 6300)\cdot \frac{1}{6302} \\
        =&\Pr[\bin(m, \alpha / 2) > s] + m\alpha / 4200 \\
        \le& \Pr[\bin(m, 2\alpha / 3) > s].
    \end{align*}
    Thus, we got
    \begin{align*}
        \max\left(\Pr_{x\sim\D_\X}[x\in[l, mid - 1]], \Pr_{x\sim\D_\X}[x\in[mid + 1, r]]\right) &= \max\left(\dis_{\D_\X}(f_{l-1}, f_{mid-1}),\dis_{\D_\X}(f_{mid+1}, r)\right) \\&\le 2\alpha/3.
    \end{align*}
    By the union bound, the above happens with probability at least $1-\beta$, which completes the proof.
\end{proof}

\subsection{Proof of Theorem~\ref{thm:threshold}}

\begin{proof}
    As proved by~\citet{ghazi2023user}, private realizable learning requires $\Omega\left(\frac{\repd(\C)}{\varepsilon} + \frac{\repd(\C)}{m\alpha\varepsilon}\right)$ users and private agnostic learning requires $\Omega\left(\frac{\vc(\C)}{\sqrt{m}\alpha\varepsilon} + \frac{\vc(\C)}{m\alpha^2}\right)$. By the fact that the class of thresholds has VC dimension $1$ (see, e.g.,~\citep{shalev2014understanding}) and representation dimension $\Theta(\log|\X|)$~\citep{feldman2014sample}, the lower bound is proved.

    For the upper bound, we assume that $m \le 1/\alpha^2$. The bound for $m > 1/\alpha^2$ follows by discarding the extra samples.

    Let $\eta = \inf_{c\in\C}\err_{\D}(c)$ and $T = \log_{3/2}\left(\frac{2}{\alpha}\right) = O(\log(1/\alpha))$.We first run the algorithm in Lemma~\ref{lem:eta} with parameter $\varepsilon / 2$. For $n\ge\widetilde{O}\left(\frac{1}{m\alpha^2} + \frac{1}{\sqrt{m}\alpha\varepsilon}\right)$, it returns some $\hat{\eta}$ s.t. $\vert\hat{\eta} - \eta\vert\le \alpha / 6$ with probability $1 - \beta / 4$. By Lemma~\ref{lem:bin}, there exists some $t$ s.t.
    \begin{equation*}
        \Pr[\bin(m, \hat{\eta} + \alpha /6) > t] + \sqrt{m}\alpha / 4200 \le \Pr[\bin(m, \hat{\eta} + \alpha / 3) > t]
    \end{equation*}
    
    Let $\psi = \Pr[\bin(m, \eta) > t]$ and $\hat{\psi} = \Pr[\bin(m, \hat{\eta} + \alpha / 6)]$. Since $\eta \le\hat{\eta} + \alpha / 6$, we have $\psi \le \hat{\psi}$. Thus, for any $u$ s.t. $\err_{\D^m}^t(f_u) \le \psi + \sqrt{m}\alpha / 4200$, it holds that $\err_{D}(f_u) \le \hat{\eta} + \alpha / 3 \le \eta + \alpha /2$.

    We then run Algorithm~\ref{alg:threshold} with dataset $\z$, privacy parameter $\varepsilon / 2$, $T$ and $t$ as declared before, algorithm $\mathsf{PrivateCompare}$ as the one in Lemma~\ref{lem:median}, and $\theta = 2/3$. The privacy guarantee of this part directly follows from the property of Laplace mechanism (Lemma~\ref{lem:lap}) and the basic composition. Thus the privacy of the entire algorithm is ensured by applying the basic composition again.

    We now consider the utility. Let $l_k, r_k$ be the endpoints after the $k$-th iteration and $T'$ be the number of finished iterations during the execution. Then there are three possible cases: 
    \begin{itemize}
        \item $T' < T$ and the algorithm returns $f_{mid}$ during the $(T' + 1)$-th iteration
        \item $T' < T$ and the algorithm returns $f_{l_{T'}} = f_{r_{T'}}$
        \item $T' = T$ and the algorithm returns $f_{l_T}$. 
    \end{itemize}
    
    We first prove that \begin{equation}
    \label{equ:threshold}
        \min_{u\in\{l_k, \dots, r_k\}}\err_{\z}^t(u) \le \psi + \sqrt{m}\alpha / 12600 + k\sqrt{m}\alpha / (12600T)
    \end{equation}
    for $0\le k\le T'$ by induction.
    
    Let $n \ge \widetilde{O}\left(\frac{1}{\sqrt{m}\alpha\varepsilon}\right)$, then by Lemma~\ref{lem:lap} and the union bound, with probability $1- \beta / 4$, in each iteration the absolute value of the generated Laplace noise is no larger than $\alpha\sqrt{m} / (25200T)$. At the beginning, we have $l_0 = 0$ and $r_0 = \vert\X\vert$. By agnostic generalization (Lemma~\ref{lem:user_ag}), for $n\ge\widetilde{O}\left(\frac{1}{m\alpha^2}\right)$, we have $\left\vert\err_{\z}^t(f_u)- \err_{\D^m}^t(f_u) \right\vert\le \sqrt{m}\alpha / 12600$ for all $u\in\{0, \dots, \vert\X\vert\}$ with probability $1-\beta / 4$. This implies $\min_{u\in\{l_0, \dots, r_0\}}\err_{\z}^t(f_u) \le \psi + \sqrt{m}\alpha / 12600$. 
   
    Now suppose \eqref{equ:threshold} holds for some $k < T'$. Since the algorithm does not return during the $(k + 1)$-th iteration, we thus have
    \begin{align*}
        \min_{u\in\{l_{k + 1}, \dots, r_{k + 1}\}}\err_{\z}^t(f_u) &\le \min(v_l, v_r, v_{mid}) + \sqrt{m}\alpha / (25200T)\\ &\le \min_{u\in\{l_{k}, \dots, r_{k }\}}\err_{\z}^t(f_u) + \sqrt{m}\alpha /(25200T) +\sqrt{m}\alpha /(25200T)\\
        &\le \psi + \sqrt{m}\alpha /12600 + (k + 1)\sqrt{m}\alpha / (12600T).
    \end{align*}

    Thus, we prove~\eqref{equ:threshold}. Now, let's consider the three cases one by one. If the algorithm returns $f_{mid}$ during the $(T' + 1)$-th iteration, then we have
    \begin{align*}
        \err_{\z}^t(f_{mid}) &\le \min(v_l,v_r,v_{mid}) + \sqrt{m} \alpha / (25200T)\\
        &\le \min_{u\in\{l_{T'}, \dots, r_{T'}\}}\err_{\z}^t(f_u) + \sqrt{m} \alpha / (25200T)+\sqrt{m}\alpha / (25200T)\\
        &\le \psi + \sqrt{m}\alpha / 12600+ T'\cdot \sqrt{m}\alpha / (12600T) + \sqrt{m}\alpha / (12600T)\\
        &\le \psi + \sqrt{m}\alpha / 6300.
    \end{align*}
    Otherwise if the algorithm returns $f_{l_{T'}}$ due to $l_{T'} = r_{T'}$ after $T'$ iterations, we also have \begin{equation*}
        \err_{\z}^t(f_{l_{T'}}) \le \psi + \sqrt{m}\alpha / 12600 + T'\cdot \sqrt{m}\alpha / (12600T) \le \psi + \sqrt{m}\alpha / 6300.
    \end{equation*}
    Thus by agnostic generalization, for both of these two cases, the algorithm returns some $f_{u_0}$ s.t. \begin{equation*}
        \err_{\D^m}(f_{u_0}) \le\psi + \sqrt{m}\alpha / 6300 + \sqrt{m}\alpha / 12600\le \psi + \sqrt{m}\alpha / 4200.
    \end{equation*}
    This implies $\err_{\D}(f_{u_0}) \le \eta + \alpha / 2 < \eta + \alpha$.

    Now consider the remaining case that the algorithm returns $f_{l_T}$ after $T$ iterations. The same argument shows that there exists some $l_T\le u'\le r_T$ s.t. $\err_{\D}(f_{u'})\le \eta + \alpha / 2$. By Lemma~\ref{lem:median} and the union bound, we know that for $n\ge\widetilde{O}\left(\frac{\log\vert\X\vert}{\varepsilon} + \frac{\log\vert\X\vert}{m\alpha\varepsilon}\right)$, it holds with probability $1-\beta / 4$ that $\Pr_{x\sim\D_\X}[x\in[l_T, r_T]] \le \left(\frac{2}{3}\right)^{T} \le \alpha / 2$. Thus,
    \begin{equation*}
        \err_{\D}(f_{l_T}) \le \err_{\D}(f_{u'}) + \Pr_{x\sim\D_\X}[x\in[l_T, r_T]] \le \eta + \alpha /2 + \alpha / 2 \le \eta + \alpha.
    \end{equation*}
    All in all, when $n \ge\widetilde{O}\left(\frac{\log\vert\X\vert}{\varepsilon} + \frac{\log\vert\X\vert}{m\alpha\varepsilon} +\frac{1}{\sqrt{m}\alpha\varepsilon} + \frac{1}{m\alpha^2}\right)$, the algorithm is an $(\alpha, \beta)$-agnostic learner with $\varepsilon$-differential privacy.
\end{proof}